\documentclass[11pt]{article}

\usepackage[margin=1in]{geometry}
\usepackage{amsmath,amssymb,amsthm}
\usepackage{booktabs}
\usepackage{graphicx}
\usepackage{xcolor}
\usepackage{pifont}
\usepackage[round]{natbib}
\usepackage[colorlinks=true,linkcolor=blue,citecolor=blue,urlcolor=blue]{hyperref}

\newtheorem{proposition}{Proposition}
\newtheorem{lemma}{Lemma}
\newtheorem{corollary}{Corollary}
\theoremstyle{remark}
\newtheorem{remark}{Remark}

\title{An Enclosed Mode Is a Gauge Choice: Topology Relative to Reach
in Certified Code World Models}

\author{Javier Aguilar Mart\'in\\ AGILabs (\href{https://javieraguilar.ai}{javieraguilar.ai})}

\date{}

\begin{document}

\maketitle

\begin{abstract}
A code world model accepted by a sampling gate can be exactly right on
everything the gate can see and arbitrarily wrong beyond it. We characterize
what a certified model can know --- and what its errors can cost --- when the
omitted structure is a topologically nontrivial region: an annular freeze
mode enclosing an unreachable interior. The \emph{gate quotient} makes the
question precise: acceptance-with-certainty determines the model exactly up
to the reachable query set, so everything beyond reach is gauge. On a minimal
ring instrument we prove the extreme case --- a wrong-topology filled-disc
artifact is unfalsifiable by \emph{any} sampling gate and bitwise harmless at
play --- and measure, with real large-language-model (LLM) synthesis across
three model families,
how a single knob (an angular channel of width $\gamma$) walks the same
artifact through three regimes: unfalsifiable-and-harmless,
falsifiable-and-costly, and instantly falsified. Three principles organize
the empirics. First, \textbf{danger is topology relative to reach, not
topology}: opening a channel the planner can use collapses the blind model's
exploitation (\texttt{play\_cost} $1.09 \to {\sim}0$ over a knee at
$\gamma \approx 0.1$), while a hidden channel with the \emph{same} first
Betti number keeps it at full strength ($1.12$). Second, \textbf{repair is
parameter-bound and sensor-bound}: from outside evidence no family recovers
the region (superstitious point fits); from inside, models pose the right
topology but cannot pin its parameters to the gate's precision --- and the
persistent-homology summary that guides them has a measurable,
geometric rather than budgetary resolution limit, and the posed
topology tracks that summary's wrong $\hat\beta_1$ rather than the
truth.
Third, \textbf{mitigation must match the dimension and direction of the
model's error}: point distrust-fences, a zero-dimensional cover, fail
against the one-dimensional closed boundary at their calibrated radius,
while a fence matched to the boundary's dimension and \emph{persisted}
across episodes collapses the exploitation to a two-lesson transient
($\mathrm{play\_cost}$ $0.999 \to 0.058$, truth-equal returns from the
second episode on). Against an \emph{invented} mode every fence fires
constantly and recovers nothing, and the dual certificate (freedom
points that locally un-freeze imagination where the model was refuted as
too pessimistic) collapses that failure symmetrically
($1.769 \to 0.029$): the two wrongnesses need opposite defenses, each
priced by its failure's lie rate. In $n$ dimensions the enclosing shell makes the identifiability
event near-certain ($r(n) \to 0$) while the danger stays fully exploitable
--- rarity and reachability are independent knobs of the danger law.
\end{abstract}

\section{Introduction}
\label{sec:intro}

In the Code World Model (CWM) paradigm, a language model synthesizes an
executable world model that a classical planner searches, and the model is
accepted when it reproduces sampled transitions. Two companion papers
established that this gate certifies the wrong thing, following a
quantitative law, $\mathrm{danger} = \mathrm{play\_cost} \times
(1-\mathrm{rarity})^N$: in discrete games \citep{aguilar2026verified}, and in
continuous hybrid systems, where an omitted mode is a rare rule and the
danger collapses to pure identifiability \citep{aguilar2026omitted}. The
continuous paper ended on a sharpened negative: repair-from-data is
geometry-dependent, and its two ablations located the mechanism in a
\emph{template prior} --- the synthesizer selects low-descriptive-complexity
region forms (1D thresholds, radial balls) instead of inducing the boundary,
flattening curves and curving flats alike. Curvature is not the axis. This
paper takes up the axis that survived: \textbf{topology} --- and shows that
the right formulation is not the topology of the mode but its topology
\emph{relative to what the certifying process and the planner can reach}.

Why an enclosed mode? Because it is the minimal setting where the
certification question and the topology question are the same question. An
annular freeze band --- a fence with an inside --- is the simplest region
whose defining feature, the hole, is invisible from one side: every
trajectory that approaches from outside stops at the outer wall, and
whether anything is behind it is, in a sense this paper makes exact,
\emph{not a property of the data at all}. Fences, containment shells, and
geofenced no-go regions are the practical shape of safety-critical hybrid
modes; a certification pipeline that cannot distinguish a fenced void from
a fenced hazard --- or a fence from a filled wall --- is certifying less
than it appears to. The question is what the gate's acceptance actually
pins down, and the answer turns out to be geometric: everything up to
reach, nothing beyond it, with the boundary between the two movable by
knobs that never touch the mode's own physics.

The instrument is deliberately minimal (Section~\ref{sec:instrument}): a 2D
thrust-and-drag plant with an annular freeze band (inner radius $3.5$, outer
$5.0$) wrapping a high-reward ``phantom'' lode, with three pre-registered
knobs --- an angular channel of width $\gamma$ (closed ring at $\gamma = 0$),
the channel's orientation (facing the start or hidden behind the lode), and
the start side (outside or inside the hole). Everything the paper claims
follows from what these knobs do to the \emph{reachable query set}
$\mathcal{R}$.

\paragraph{Contributions.}
\begin{enumerate}
\item \textbf{The gate quotient and the unfalsifiable-and-harmless theorem}
  (Section~\ref{sec:theory}). Acceptance-with-certainty determines a model
  exactly up to $\mu_{\mathrm{query}}$-null modification on $\mathcal{R}$;
  beyond $\mathcal{R}$ is gauge
  (Proposition~\ref{prop:quotient}). On the closed ring, the wrong-topology
  filled-disc artifact is unfalsifiable by any sampling gate \emph{and}
  bitwise planner-equivalent to the truth
  (Proposition~\ref{prop:harmless}, confirmed episode-for-episode).
  Monotonicity, continuity, and positivity of the rarity curves in $\gamma$
  round out the provable core (Section~\ref{sec:gammacurves}, stated and
  proved, one positivity witness machine-checked); the two monotonicity
  statements hold up to an explicit \emph{funnel defect}
  (Corollary~\ref{cor:funneldefect}) that we bound by measurement,
  $17\times$ below the effect --- unconditionally they are out of reach,
  and a measured certificate shows why no pathwise proof can exist
  (Remark~\ref{rem:certificate}).
\item \textbf{The three-regime mechanism, without LLMs}
  (Section~\ref{sec:mechanism}). One artifact walks through
  unfalsifiable-and-harmless ($\gamma = 0$), falsifiable-and-costly (facing
  $\gamma > 0$), and instantly-falsified (inside start) as the knobs move.
  Two readings matter later: an \emph{invented} mode is exploited
  \emph{below random} from inside ($\mathrm{play\_cost}$ $1.77$ --- the dual
  of phantom-free-space exploitation), and the hidden channel is
  observationally \emph{identical} to the closed ring --- same $\beta_1$
  change, no observable consequence.
\item \textbf{LLM synthesis on the closed ring, three families}
  (Section~\ref{sec:synthesis}). The identifiability event is
  family-independent (GPT-5.x, Qwen, Claude: every mode-absent sample yields
  a certified, blind, exploited model at $\mathrm{play\_cost} \approx
  1.12$). None of the three tested families repairs the ring from outside
  evidence --- a behavioral
  audit shows the gate-passing ``repairs'' are measure-zero point fits.
  From inside, models \emph{pose} the right topology but repair is
  \textbf{parameter-bound}: the single gate-certified recovery (1/20) used the
  gauge-free disc-complement form whose one parameter is the true round
  radius anchored in the reward specification --- and the strongest
  cross-family repairer (Claude, 3/3) used exactly the same form.
\item \textbf{The open-ring arm: danger is topology relative to reach}
  (Section~\ref{sec:openring}). Pre-registered sweep over $\gamma$: the
  blind model's exploitation collapses over a knee at $\gamma \approx 0.1$
  (channel arc-width comparable to the planner's step) when the channel
  faces the start, and \emph{persists at full strength} ($1.116$) at the
  same $\gamma$ when the channel is hidden --- identical $\beta_1$, opposite
  danger. Gate-pass from inside stays ${\approx}0$ at every $\gamma$, and
  the two certified passes at wide $\gamma$ are \emph{certified-wrong}:
  mode-blind on the probes, harmless only because the open ring no longer
  obstructs.
\item \textbf{Mitigation obeys a covering law --- and the defense that
  works pays it once} (Section~\ref{sec:mitigation}). The companion
  paper's distrust-fence defense, verbatim, does nothing on the ring at its
  calibrated radius: point fences are a zero-dimensional cover of a
  one-dimensional boundary, and the planner re-crosses through unfenced
  imagined corridors (cost = boundary measure over fence radius). Matching
  the fence's dimension to the boundary's (violations linked into
  tangentially-extended segments) \emph{and persisting it across episodes}
  collapses the exploitation to a two-lesson transient:
  $\mathrm{play\_cost}$ $0.999 \to 0.058$, with returns equal to the
  truth planner's from episode 2 on --- the fence engineers the blind model
  back into the truth's gate-equivalence class on the operative side.
  Against the \emph{invented} mode every distrust variant fires constantly
  and changes nothing --- and the \emph{dual} defense (freedom patching:
  un-freeze imagination, via the pinned integrator, where the model was
  refuted as too pessimistic) collapses that failure at once
  ($1.769 \to 0.029$, near-truth returns from episode 1). The two
  wrongnesses need opposite defenses, and each defense's cost is set by its
  failure's lie rate: false obstructions refute themselves at every step;
  false freedoms only at the rare boundary.
\item \textbf{The evidence sensor has finite resolution --- and the loop
  follows its report} (Section~\ref{sec:sensor}). The pre-registered
  persistent-homology summary reports $\hat\beta_1 = 1$ for every channel
  narrower than ${\sim}2$ arc-units (flip at $\gamma \approx 1.8$) even
  though the true $\beta_1 = 0$ for all $\gamma > 0$; the posed artifact
  topology tracks the \emph{guidance}, not the truth. A pre-registered
  flipped-summary crossover ($60$ paired seeds) is directionally
  consistent with the claim line steering the artifact ($9{:}2$
  discordant split) but not significant ($p = 0.065$), so the causal
  reading stays unearned: what is measured is the association. On the
  $n$-dimensional shell, inside-start evidence recovers $\beta_{n-1}$ up to
  $n = 5$ where outside evidence recovers it at \emph{no} $n$ --- recovery
  is start-governed, and the Niyogi--Smale--Weinberger density bound makes
  the outside failure computable in advance.
\item \textbf{Dimension as the rarity knob; the two axes are independent}
  (Section~\ref{sec:ndim}). On ShellField-$n$ the contact rarity collapses
  geometrically (measured factor $0.411$ per dimension; the exponential
  rate is proved under the isotropic action interface, and an explicit
  $4h/(n\kappa^2)$ bound for the instrument's box thrust; contacts
  $1/600$ at $n = 4$ and at $n = 6$, a censored zero at $n = 5$ ---
  the $600$-rollout calibration is at its own resolution floor from
  $n = 4$ on, and the rate above is measured on $10{,}000$ rollouts per
  dimension), making mis-synthesis
  near-certain, while a competent vector planner is exploited at
  $\mathrm{play\_cost} \approx 1.0$ at \emph{every} $n \le 6$ --- rarity
  (synthesis axis) and reachability (play axis) are independent knobs, and
  a high-$n$ enclosed mode maxes out both.
\item \textbf{A behavioral audit that separates absent, wrong, and
  unidentifiable structure} (used throughout): probe each
  artifact's \texttt{step()} on a state grid, classify the shape of its
  deviation-from-integrator set. It reproduces the companion paper's hand
  inspection, reclassifies the inside-start failures, and separates
  ``posed the wrong structure'' from ``posed the right structure with
  unidentifiable parameters.''
\end{enumerate}

The practitioner's summary is one sentence long: \emph{a sampling gate
certifies the reachable restriction of the model and nothing else} --- so
the safety question is never ``is the model right?''\ but ``does the
place where it is wrong intersect the operative reach of a competent
planner?''\ Everything in this paper --- the harmless wrong topology, the
danger knee, the hidden-channel persistence, the parameter-bound repairs,
the mitigation asymmetry, the sensor's resolution limit, and the
dimensional rarity collapse --- is that one sentence measured from a
different direction.

\section{The instrument: a mode with an inside}
\label{sec:instrument}

\texttt{RingField2D} is the companion paper's plant with one structural
change. The state is $(x, y, v_x, v_y)$; a scalar action
$a \in [-a_{\max}, a_{\max}]$ maps to a thrust heading
$\phi = \pi a / a_{\max}$; velocities integrate with gain $3.0$ and drag
$0.3$ at $dt = 0.1$ (the same semi-implicit integrator family as
PatchField2D, so every planner and every synthesis-protocol constant carries
over unchanged). The reward is two radial sigmoid lodes: a small
\emph{real} one (amplitude $0.3$) behind the start at $(-6, 0)$, and a
large \emph{phantom} one (amplitude $1.0$) at $c = (12, 0)$. The structural
change is the mode: an \textbf{annular freeze band}
$A = \{\, r_{\mathrm{in}} \le \lVert p - c \rVert \le
r_{\mathrm{out}} \,\}$ with $r_{\mathrm{in}} = 3.5$ and
$r_{\mathrm{out}} = 5.0$, wrapping the phantom lode. A step whose next
position lands in $A$ freezes at the \emph{previous} position with zero
velocity --- the same inelastic-stop semantics as the companion paper's
patches, but now the region has an inside: the ring separates the plane, and
its defining feature ($\beta_1 = 1$, the hole) is exactly the part no
outside observer ever touches.

Three pre-registered knobs deform the relationship between the mode and what
is reachable, without touching the mode's local physics:
\begin{itemize}
\item \textbf{Channel width $\gamma$} (radians): an angular gap cut from
  the ring. $\gamma = 0$ is the closed ring ($\beta_1 = 1$); any
  $\gamma > 0$ is an open C-shape ($\beta_1 = 0$, contractible).
\item \textbf{Channel orientation}: centered facing the start
  (\texttt{gap\_center} $= \pi$) or hidden on the far side
  (\texttt{gap\_center} $= 0$). Both have identical topology at equal
  $\gamma$; only the channel's position relative to the start and the
  planner's optimal path differs.
\item \textbf{Start side}: outside the ring (the default; the interior is
  then unreachable at $\gamma = 0$, Lemma~\ref{lem:crossing}) or inside
  the hole (the interior evidence regime).
\end{itemize}
The defaults were frozen once at calibration and never tuned per cell:
episode horizon $h = 80$ steps; thickness $w = 1.5$ exceeds the maximal
per-step displacement
$\Delta \le 1.0$ (Section~\ref{sec:theory}), so the no-jump-over lemma
holds with margin for real and imagined rollouts alike. The outside contact
rarity at the defaults is $r = 0.0312$ per rollout ($937/30{,}000$
calibration rollouts; the mechanism grid's own $400$-rollout realization
reads $18/400 = 0.045$, statistically compatible). The synthesis protocol is
the companion paper's, verbatim: $N = 40$ uniform-random rollouts as
evidence and gate, the integrator pinned in the contract, exactness
tolerance $\varepsilon = 10^{-9}$, at most $5$ refine iterations, and
\texttt{play\_cost} the paired-seed normalized regret of planning on the
artifact and executing in the truth. Section~\ref{sec:ndim} generalizes the
instrument to \texttt{ShellField}-$n$ --- the same two-lode geometry with a
spherical shell $S^{n-1}$ in the first two coordinates and a thrust-vector
action $\vec a \in [-1, 1]^n$ (norm-capped) --- so that the ambient
dimension $n$ becomes a fourth knob.

\section{Theory: the gate quotient}
\label{sec:theory}

This section establishes the paper's governing equivalence: a sampling
gate identifies dynamics only on its reachable query set, so everything
beyond reach is gauge. It then specializes that quotient to the enclosed
ring --- where the extreme case is a theorem --- and shows how the
channel width reopens identifiability from an exact zero.

\paragraph{Reachable queries define the gate's equivalence class.}
Consider deterministic dynamics $f : S \times A \to S$ on
$S \subseteq \mathbb{R}^d$, an initial distribution $\mu_0$, an episode
horizon
$h$, and a \emph{gate policy} $\rho$ selecting actions (uniform-random in
our gates, though nothing below needs that). The \textbf{reachable query
set} $\mathcal{R}(f, \mu_0, \rho, h)$ is the set of pairs $(s, a)$ at
which some length-${\le}h$ trajectory under $(f, \mu_0, \rho)$ queries
$f$ with positive probability --- taken in continuous spaces as the support
of the induced occupation measure.
A \emph{sampling gate} of any size $N$ draws trajectories under
$(f, \mu_0, \rho)$ and accepts a candidate $\hat f$ iff $\hat f$
reproduces every queried transition (within any tolerance, including
exactly).

\begin{proposition}[gate quotient]
\label{prop:quotient}
Let $\mathcal{R}(f, \mu_0, \rho, h)$ be the reachable query set of the truth
$f$ under the gate policy $\rho$, and let $\hat f$ be any model with
$\hat f|_{\mathcal{R}} = f|_{\mathcal{R}}$. Then for every $N$ the sampling
gate accepts $\hat f$ with probability~1, and the trajectory law of any
policy whose queries stay in $\mathcal{R}$ is identical under $f$ and
$\hat f$. Conversely, let $\mu_{\mathrm{query}}$ be the gate's query-occupation
measure (the expected number of queries per rollout landing in each
measurable set of $S \times A$) and let $D$ be the set where $\hat f$
disagrees with $f$ beyond the gate's tolerance. Acceptance with
probability~$1$ at every $N$ forces $\mu_{\mathrm{query}}(D) = 0$: the
model may differ from $f$ only on a $\mu_{\mathrm{query}}$-null subset of
$\mathcal{R}$ and on the gauge complement
$\mathcal{G} = (S \times A) \setminus \mathcal{R}$, which no sample drawn
under $\rho$ can falsify. In particular the extension class
$E(f) = \{ \hat f : \hat f|_{\mathcal{R}} = f|_{\mathcal{R}} \}$ is
accepted with certainty, and is the full accepted class up to
$\mu_{\mathrm{query}}$-null sets.
\end{proposition}

\begin{proof}
By induction on the step index, a trajectory under $(f, \mu_0, \pi)$ whose
queries lie in $\mathcal{R}$ is a trajectory under $(\hat f, \mu_0, \pi)$
with the same realizations: the state after step $t$ is a function of
$\mu_0$, the action sequence, and the queried values of $f$, which agree
with $\hat f$'s on $\mathcal{R}$. Gate draws use $\pi = \rho$, whose
queries lie in $\mathcal{R}$ by definition (up to a null set), so acceptance
statistics coincide; $f \in E(f)$ is accepted almost surely, hence so is
every member. For the converse, let $p$ be the probability that a single
gate rollout queries $D$ at least once. Rollouts are i.i.d., so the
size-$N$ gate accepts $\hat f$ with probability $(1 - p)^N$, and
acceptance with certainty at every $N$ forces $p = 0$. Finally $p = 0$
iff $\mu_{\mathrm{query}}(D) = 0$: the per-rollout count of $D$-queries
is a nonnegative random variable, so its expectation vanishes iff the
count is almost surely zero.
\end{proof}

\begin{remark}[relation to the companion identifiability proposition]
The companion paper's identifiability result conditions on the finite-sample
event ``the mode region was missed,'' of probability $(1 - r)^N \to 0$ when
$r > 0$. Proposition~\ref{prop:quotient} is its structural,
$N$-independent limit: on $\mathcal{G}$ the miss probability is $1$ for
\emph{every} $N$. The companion's prior caveat --- a prior or the
specification could still supply the mode --- is exactly the statement that
$E(f)$ is not a singleton.
\end{remark}

\begin{lemma}[metric crossing]
\label{lem:crossing}
If positions move in steps of norm at most $\Delta$ and the annulus
$A = \{ r_{\mathrm{in}} \le \lVert p - c \rVert \le r_{\mathrm{out}} \}$ has
thickness $w = r_{\mathrm{out}} - r_{\mathrm{in}} > \Delta$, then any
trajectory from outside to the open interior visits $A$ at an intermediate
step. Consequently, at $\gamma = 0$ with freeze-on-entry dynamics the open
interior is reach-null: the inner disc $\times$ all actions lies in
$\mathcal{G}$.
\end{lemma}

\begin{proof}
$g(t) = \lVert p_t - c \rVert$ changes by at most $\Delta$ per step
(distance-to-center is 1-Lipschitz). Let $t^*$ be the first index with
$g(t^*) < r_{\mathrm{out}}$. If $g(t^*) < r_{\mathrm{in}}$ then
$g(t^*{-}1) \ge r_{\mathrm{out}}$ forces a step longer than
$w > \Delta$, a contradiction; so $p_{t^*} \in A$. Under the freeze
semantics the first landing in $A$ is replaced by the previous (outside)
position with zero velocity, so $g$ never drops below $r_{\mathrm{in}}$;
induct.
\end{proof}

At the frozen defaults the hypothesis holds with margin: speed obeys
$\lVert v_{t+1} \rVert \le (1 - \mathrm{drag}\,dt)\lVert v_t \rVert +
\mathrm{gain}\,dt$, so $\lVert v \rVert \le \mathrm{gain}/\mathrm{drag}
= 10$ and $\Delta \le 1.0 < w = 1.5$ --- for real \emph{and} imagined
rollouts, since planners use the same integrator and action clamp. Measured:
200 calibration rollouts visit zero interior states while reaching the ring
itself.

\begin{corollary}[evidence equivalence of disc and annulus from outside]
\label{cor:evidence}
From outside starts no landing ever falls at
$d < r_{\mathrm{in}}$, so the disc mode
$\lVert p - c \rVert \le r_{\mathrm{out}}$ and the annulus mode fire on
exactly the same steps of every realization: the contact processes --- and
hence \emph{all} evidence any gate, repair loop, or topological summary
extracts --- are pathwise identical. (Measured as an exact row-for-row
equality of the two contact clouds at every sample size.) The companion
paper's dimensional reduction is \emph{rational given the evidence}:
outside data cannot even pose the disc-versus-annulus question.
\end{corollary}

\begin{remark}[metric, not topological --- where algebraic topology is
earned]
Lemma~\ref{lem:crossing} uses only that distance-to-center is 1-Lipschitz;
it works verbatim for round shells $S^{n-1}$ in $\mathbb{R}^n$ and needs no
homology. Genuine algebraic topology is \emph{earned} exactly where this
proof dies: non-round separators (Jordan--Brouwer to even define
``inside'') and non-separating modes (winding and linking obstructions ---
for the solid torus this is carried out in Section~\ref{sec:tube}: the
linking dichotomy). The round ring deliberately sits on the metric side of
that boundary; the topology enters through what the gauge region and its
boundary class organize, not through the crossing proof.
\end{remark}

\begin{remark}[machine-checked]
The paper's deterministic results are formalized in Lean~4 over mathlib
(\texttt{formal/}, library \texttt{Paper3Ring}; no \texttt{sorry}), each at
the level its proof actually operates: Lemma~\ref{lem:crossing} in both
halves over an arbitrary pseudometric space, with the constants paragraph
above ($\lVert v \rVert \le \mathrm{gain}/\mathrm{drag}$ through freeze
events, $\Delta = 1.0 < w = 1.5$ by \texttt{norm\_num});
Corollary~\ref{cor:evidence} as a pathwise identity of trajectories and
contact processes; the gate quotient (Proposition~\ref{prop:quotient}) at
realization level; Proposition~\ref{prop:harmless} composed with the
planner/environment loop --- unfalsifiability over arbitrary closed-loop
policies, and play cost exactly $0$ for any deterministic planner acting
on imagined rollouts; the firing coupling of
Proposition~\ref{prop:rmono} and the direct-entry transport of
Proposition~\ref{prop:direct}, with
Corollary~\ref{cor:funneldefect}'s defect algebra checked from its named
hypotheses; the fence and patch sufficiency theorems
(Propositions~\ref{prop:fence} and~\ref{prop:patch}) from their
coverage/dominance hypotheses; Proposition~\ref{prop:positivity} closed up
to the start box's positive probability (the 1-D tube reduction, the
in-horizon window, and channel membership via Jordan's inequality); and
every measure step of Proposition~\ref{prop:continuity}'s modulus end to
end --- the landing law's exact uniformity, the circle--strip bound with
rotation invariance as a measure statement, measurability of the freeze
trajectory in the action sequence, and the $h$-step slice/union
composition --- leaving only each application's coordinate computation.
The item-by-item map, the modelling conventions, and the exact boundary of
what remains measured are recorded in the repository's formalization
notes.
\end{remark}

\begin{proposition}[wrong topology: unfalsifiable and harmless at $\gamma = 0$]
\label{prop:harmless}
Let $f$ be the closed-ring dynamics and $\hat f_{\mathrm{fill}}$ the
filled-disc model (freeze on the whole disc
$\lVert p - c \rVert \le r_{\mathrm{out}}$). Then (i)
$\hat f_{\mathrm{fill}} \in E(f)$: every sampling gate accepts it with
probability~1; and (ii) any deterministic planner whose imagined rollouts
start at real (outside) states produces \emph{identical} real trajectories
under $\hat f_{\mathrm{fill}}$ and under $f$ ---
$\mathrm{play\_cost} = 0$ exactly.
\end{proposition}

\begin{proof}
(i) is Proposition~\ref{prop:quotient} plus Lemma~\ref{lem:crossing}: the
two models disagree only on next-states in the open interior, which lies
outside $\mathcal{R}$. (ii) Imagined rollouts from an outside state under
either model freeze at the same first-annulus landing --- the models agree
on $A$ and outside, and by Lemma~\ref{lem:crossing} (applied to imagined
paths, which use the same integrator and step bound) imagination never
produces an interior query where they differ. Every candidate action
sequence therefore receives the same imagined return under both models; a
deterministic planner selects the same action at every real step, and the
real environment does the rest.
\end{proof}

Paired-seed model-predictive-control (MPC) episodes confirm
Proposition~\ref{prop:harmless} \emph{bitwise}: return, final state, and
contact are identical under $f$ and $\hat f_{\mathrm{fill}}$, seed for
seed.

\begin{remark}[the triad]
\label{rem:triad}
Proposition~\ref{prop:harmless} is the cleanest statement of a three-way
split this paper keeps separating: \emph{certifiability},
\emph{correctness}, and \emph{consequence} are three different things. At
$\gamma = 0$ the filled-disc artifact is wrong, certified, and costless;
$\gamma > 0$ continuously converts the same wrongness into
$(1 - r(\gamma))^N$-gated danger --- $E(f)$ shrinks and
$\hat f_{\mathrm{fill}}$ exits it. One knob walks one artifact through
all three regimes (Section~\ref{sec:mechanism}).
\end{remark}

\subsection{The $\gamma$-curves and the query lower bound}
\label{sec:gammacurves}

Fix one probability space for all $\gamma$: a single i.i.d.\ action
sequence and initial state drive the dynamics at every channel width
(common random numbers), so that $\gamma_1 < \gamma_2$ gives nested mode
regions $A(\gamma_2) \subseteq A(\gamma_1)$ with difference slivers
$D = A(\gamma_1) \setminus A(\gamma_2)$. Write $r(\gamma)$ for the
per-rollout contact rate and $r_{\mathrm{int}}(\gamma)$ for the
interior-entry rate.

\begin{lemma}[divergence localization]
\label{lem:divergence}
Under the coupling, the $\gamma_1$- and $\gamma_2$-trajectories coincide
up to (and excluding) the first step whose landing falls in $D$; hence for
any trajectory event, $|P_{\gamma_1} - P_{\gamma_2}| \le P(\text{some
landing in } D)$.
\end{lemma}
\begin{proof}
Before that step every landing is either outside $A(\gamma_1)$ (both
trajectories free, same next state) or in $A(\gamma_2)$ (both freeze at
the same previous position); induct. The bound is the coupling inequality.
\end{proof}

\begin{proposition}[$r$ is nonincreasing in $\gamma$]
\label{prop:rmono}
Contact events are nested pathwise:
$\mathrm{fire}(\gamma_2) \subseteq \mathrm{fire}(\gamma_1)$, so
$r(\gamma_2) \le r(\gamma_1)$.
\end{proposition}
\begin{proof}
If the $\gamma_2$-trajectory first fires at step $t$ with no $D$-landing
before $t$, Lemma~\ref{lem:divergence} makes the trajectories agree
through $t$ and the landing lies in $A(\gamma_2) \subseteq A(\gamma_1)$;
if some $D$-landing occurs at $s \le t$, a landing in $D$ \emph{is} a
landing in $A(\gamma_1)$. ($0$ violations in $44{,}000$
common-random-number (CRN) checks.)
\end{proof}

\begin{proposition}[continuity of the $\gamma$-curves, with explicit modulus]
\label{prop:continuity}
For every event determined by the trajectory --- in particular for
$q = r$ and $q = r_{\mathrm{int}}$ --- and all
$0 \le \gamma \le \gamma' \le 2\pi$ with $\varepsilon = \gamma' - \gamma$,
\[
  |q(\gamma) - q(\gamma')|
  \;\le\; h \cdot \sqrt{ r_{\mathrm{out}}\,\varepsilon /
    (\mathrm{gain}\cdot dt^2) }
  \;\approx\; 1033\,\sqrt{\varepsilon}
  \quad \text{at the defaults},
\]
so both curves are uniformly H\"older-$\tfrac12$ on $[0, 2\pi]$; in
particular $r_{\mathrm{int}}(\gamma) \le 1033\sqrt{\gamma} \to
r_{\mathrm{int}}(0) = 0$.
\end{proposition}
\begin{proof}
The one-step proposed landing from any state is \emph{exactly} uniform on
a circle of radius $R_L = \mathrm{gain}\cdot dt^2$ centered at a
state-determined drift point: the integrator gives
$x' = x + (1 - \mathrm{drag}\,dt)\,v_x\,dt +
\mathrm{gain}\cdot dt^2\cos\phi$ (same for $y$) with
$\phi = \pi a / a_{\max}$ uniform on $[-\pi, \pi]$ when $a$ is uniform. Each of $D$'s two slivers
lies in a strip of width $r_{\mathrm{out}}\varepsilon/2$ around its
bisector line, and a circular-uniform law gives a width-$w$ strip mass at
most $\sqrt{w/(2R_L)}$ (the arcsine of a length-$w/R_L$ interval; the
worst case is tangency, where the bound is attained up to the factor
$\pi/\sqrt2$). Condition on the state, apply this per step and sliver,
and sum over $h$ steps via Lemma~\ref{lem:divergence}. The exponent
$\tfrac12$ is not pessimism: a drift center at distance exactly $R_L$
from a sliver line has single-step hitting probability
$\ge \sqrt{w/R_L}/\pi$, and the measured divergence probability scales as
$\varepsilon^{0.30}$ over $\varepsilon \in [0.0125, 0.2]$ --- the
channel-mouth (funnel) states occupy this tangency band and saturate per
episode --- so a linear bound on the coupling's divergence term is
refuted at the measured scale. Full statements, the linear
off-tangency refinement, and machine checks are in the repository's
theory notes (Lemmas~S/A/W, Theorems~T4/T4$'$).
\end{proof}

\begin{proposition}[direct interior entries are pathwise monotone]
\label{prop:direct}
Call an entry \emph{direct} if the trajectory never lands in $A(\gamma)$
before its first interior entry. For $\gamma_1 < \gamma_2$,
$\mathrm{direct}(\gamma_1) \subseteq \mathrm{direct}(\gamma_2)$; the
direct component of $r_{\mathrm{int}}$ is nondecreasing.
\end{proposition}
\begin{proof}
A direct-at-$\gamma_1$ trajectory's pre-entry landings avoid
$A(\gamma_1) \supseteq A(\gamma_2)$; by the
Lemma~\ref{lem:divergence} induction it is unchanged under $\gamma_2$ and
still direct.
\end{proof}

\begin{corollary}[monotonicity up to the funnel defect]
\label{cor:funneldefect}
Write $r_{\mathrm{int}} = d + f$, splitting entries into direct and
funnel-assisted. For $\gamma_1 < \gamma_2$,
$r_{\mathrm{int}}(\gamma_2) \ge r_{\mathrm{int}}(\gamma_1) -
f(\gamma_1)$; in particular any violation of M1
($r_{\mathrm{int}}$ nondecreasing) or M2
($r_{\mathrm{int}}(\gamma) \le r_{\mathrm{int}}(2\pi)$) is at most the
funnel mass at the smaller gap, and is exactly zero where $f = 0$.
\end{corollary}
\begin{proof}
$r_{\mathrm{int}}(\gamma_2) \ge d(\gamma_2) \ge d(\gamma_1) =
r_{\mathrm{int}}(\gamma_1) - f(\gamma_1)$, the middle step by
Proposition~\ref{prop:direct}. At $\gamma = 2\pi$ there is no mode, so
$f(2\pi) = 0$ and $r_{\mathrm{int}}(2\pi) = d(2\pi)$.
\end{proof}

\begin{remark}[the defect is measured, and small]
\label{rem:certificate}
Full monotonicity admits \emph{no} pathwise proof: probe seed $50543$
enters the interior at $\gamma = 0.4$ but not at $\gamma = 0.6$ (nor at
$2\pi$). Corollary~\ref{cor:funneldefect} explains rather than merely
records it --- Proposition~\ref{prop:direct} forbids a direct violation,
so the counterexample must be funnel-assisted, which it is. It also
bounds the damage. Measuring the defect over $50{,}000$ common-random
rollouts per gap: the funnel mass peaks at $22/50{,}000$ around
$\gamma \in [0.6, 0.9]$ and vanishes at both ends (nothing to enter at
$\gamma = 0$; nothing to freeze against past saturation), giving a
certified slack of $6.7 \times 10^{-4}$ against an effect size of
$1.12 \times 10^{-2}$ --- a factor of $17$, with zero empirical
violations and Proposition~\ref{prop:direct} confirmed across
$550{,}000$ pathwise comparisons. So M1 and M2 hold as theorems up to a
defect an order of magnitude below the phenomenon; for
$\gamma \ge 3.2$ the funnel mass measures $0/50{,}000$ (a censored zero,
Wilson upper $7.7 \times 10^{-5}$), so there the defect is at most that
bound --- exactness itself holds wherever $f = 0$
(Corollary~\ref{cor:funneldefect}(c)), which a censored zero bounds but
does not prove. What remains open
is an \emph{a priori} bound on the funnel mass; the one pointwise
estimate that would have given unconditional monotonicity was
\emph{refuted} by a freeze-rescue effect ($91/91$ CI-separated
violations). Nothing below load-bears on M1/M2 (Corollary~\ref{cor:knob}).
\end{remark}

\begin{proposition}[positivity: $r_{\mathrm{int}}(\gamma) > 0$ for every
$\gamma > 0$, facing channel]
\label{prop:positivity}
At the frozen defaults with the channel facing the start,
$r_{\mathrm{int}}(\gamma) > 0$ for every $\gamma > 0$.
\end{proposition}
\begin{proof}[Proof (witness tube, compact)]
Condition on $|y_0| \le \eta(\gamma) := (3.5/8)\gamma \wedge 0.4$ (an
event of positive probability under the uniform start). The constant
action $a \equiv 0$ drives the state due east along $y = y_0$; a landing
inside the band at radius $d \ge r_{\mathrm{in}}$ has angular offset
$\theta$ from $\pi$ with
$\sin\theta = |y_0|/d \le |y_0|/r_{\mathrm{in}} \le \gamma/8$. For
$\gamma \le \pi$, Jordan's inequality gives
$\gamma/8 < \gamma/\pi \le \sin(\gamma/2)$, and since both angles lie in
$[0, \pi/2]$ this yields $\theta < \gamma/2$; for $\gamma > \pi$ the
bound is immediate, $\theta \le \pi/2 < \gamma/2$. The landings lie in
the channel, so
the witness path is freeze-free with clearance
$c(\gamma) > 0$ from $A(\gamma)$ and first lands past the band at depth
$< r_{\mathrm{in}} - c$. On the freeze-free tube the $h$-step flow map is
Lipschitz in the action sequence (explicit constant $L_h$ by the linear
drag recursion), so every action sequence within
$\rho = c/(2 L_h)$ of $0$ in sup norm still enters; that action tube has
probability $\rho^h > 0$ (machine-checked witness:
\texttt{test\_positivity\_witness\_tube}).
\end{proof}

\begin{corollary}[the knob is theorem-backed]
\label{cor:knob}
$r_{\mathrm{int}}(0) = 0$ exactly (Lemma~\ref{lem:crossing});
$r_{\mathrm{int}}$ is continuous (Proposition~\ref{prop:continuity});
$r_{\mathrm{int}}(\gamma) > 0$ for every $\gamma > 0$
(Proposition~\ref{prop:positivity}). The $\gamma$-knob re-opens
identifiability \emph{continuously from an exact zero}; the instrument's
claims need nothing from M1/M2.
\end{corollary}

Positivity is also where sample size earns its keep: $30{,}000$ rollouts
measure $3$ and $12$ interior entries at $\gamma = 0.05$ and $0.1$
(rates $1.0\times10^{-4}$ and $4.0\times10^{-4}$;
\texttt{results/ring2d\_rarity\_sweep.json}, unit rollout) ---
confirming Proposition~\ref{prop:positivity} at gaps where the witness
tube's own bound is astronomically small, and where any calibration of a
few hundred rollouts would typically read $0$, indistinguishable from
the theorem's exact zero at $\gamma = 0$. Zero is a theorem only at
$\gamma = 0$; everywhere else a printed zero is a sample statement.

\begin{proposition}[query lower bound]
\label{prop:query}
Let the planner select, at each real step, the first action of an imagined
argmax over a candidate set $\mathcal{C}$ under the blind model. Assume
\textup{(RG)} every candidate whose imagined path enters the phantom
basin $B(c, r_0)$ out-scores every candidate that stays outside
$B(c, r_{\mathrm{out}})$, and \textup{(C)} $\mathcal{C}$ contains an
entering candidate. Then the selected imagined path crosses the annulus
(Lemma~\ref{lem:crossing} applied to imagined paths), so the planner
queries the disagreement region with probability $1$: the companion
paper's play-cost upper bound is tight-side-active on this instrument.
\end{proposition}
\begin{proof}
By (C) an entering candidate exists; by (RG) no non-entering candidate is
the argmax; the entering path passes from outside $r_{\mathrm{out}}$ to
inside $r_0 < r_{\mathrm{in}}$ and Lemma~\ref{lem:crossing} places one of
its steps in $A$. (RG) is \emph{checkable, not behavioral}: at the frozen
defaults it holds with margin over the visited-state envelope and is
verified numerically for every arm; (C) is logged. Measured: blind
contact rate $1.0$ at
every $\gamma$.
\end{proof}

The impossibility grades split cleanly:
$r_{\mathrm{int}}(0) = 0$ is \emph{exact} (a theorem, measured $0.0000$
at $n = 4000$), whereas the \emph{hidden} channel at $\gamma > 0$ has
$r_{\mathrm{int}} > 0$ but of tube order $\rho^h$ ---
positive yet below measurement at our budgets. These are two different
\emph{grades} of impossibility, and Sections~\ref{sec:mechanism}
and~\ref{sec:openring} show the gate and the planner cannot tell them
apart while behaving identically on both. The second grade now has an
explicit \textbf{steering witness}
(\texttt{scripts/ring2d\_steering\_witness.py}): a scripted waypoint
policy that knows the hidden channel enters the interior in $100/100$
episodes at hidden $\gamma = 1.2$ \emph{within the instrument's own
horizon}, where the gate policy's measured rate is $0/100$ --- and the same
controller at $\gamma = 0$ enters $0/100$, as the theorem requires.
Certifiability is relative to the query policy $\rho$:
$E(f)$ is defined by $\rho$'s reach, and a policy that knows the channel
shrinks the gauge region. (At hidden $\gamma = 0.6$ the narrower corridor defeats the
\emph{hand} controller --- but a $10$-parameter search over the same
waypoint family finds a controller entering $100/100$ within the same
horizon, machine-checked. Reachability is jointly a property of geometry,
policy, and budget, and the policy axis is cheap to search: what the gate
policy cannot reach in $400$ rollouts, ten tuned parameters reach every
time.) The steering witness also fixes the query side: under
Proposition~\ref{prop:query}'s checkable hypotheses the competent blind
planner reaches the disagreement region with probability~$1$, and the
measured blind contact rate is $1.0$ at every calibrated $\gamma$.

\begin{figure}[t]
\centering
\includegraphics[width=0.66\linewidth]{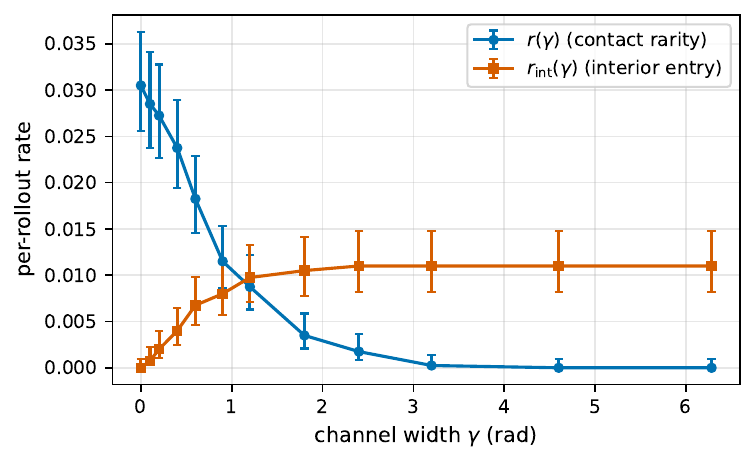}
\caption{The $\gamma$-curves, measured with common random numbers ($4{,}000$
CRN rollouts per point; Wilson 95\% CIs): the contact rarity $r(\gamma)$ is
nonincreasing (pathwise theorem; $0$ violations observed) and the
interior-entry rate $r_{\mathrm{int}}(\gamma)$ rises from its \emph{exact}
zero at $\gamma = 0$ (Lemma~\ref{lem:crossing}) --- the two grades of
impossibility are visible as the difference between a point pinned at zero by
a theorem and a curve passing through measurably positive values.}
\label{fig:gammacurves}
\end{figure}

\section{Mechanism without LLMs: three regimes, one knob}
\label{sec:mechanism}

Before any LLM enters, the mechanism grid measures what the knobs do to a
\emph{fixed} pair of wrong artifacts: the mode-blind model (the ring
deleted --- the companion papers' omission) and the filled-disc model of
Proposition~\ref{prop:harmless} (the wrong-topology invention). Each cell
of $\gamma \times \text{channel} \times \text{start}$ runs $400$
uniform-random rollouts (the gate-side falsifiability rates) and $16$
paired-seed MPC episodes per artifact (the play side).

\begin{table}[h]
\centering
\small
\begin{tabular}{lllrrrrr}
\toprule
$\gamma$ & channel & start & $r$ & $r_{\mathrm{int}}$ &
$\mathrm{disagree}_{\mathrm{fill}}$ &
$\mathrm{pc}_{\mathrm{blind}}$ & $\mathrm{pc}_{\mathrm{fill}}$ \\
\midrule
0.0 & ---     & outside & 0.0450 & 0.0000$^{\ast}$ & 0.000000$^{\ast}$ & 0.999 & 0.000$^{\ast}$ \\
0.6 & facing  & outside & 0.0350 & 0.0075 & 0.000844 & 0.022 & 0.343 \\
0.6 & hidden  & outside & 0.0450 & 0.0000$^{\dagger}$ & 0.000000$^{\dagger}$ & 0.999 & 0.000$^{\dagger}$ \\
1.2 & facing  & outside & 0.0175 & 0.0125 & 0.000625 & 0.007 & 0.220 \\
1.2 & hidden  & outside & 0.0450 & 0.0000$^{\dagger}$ & 0.000000$^{\dagger}$ & 0.999 & 0.000$^{\dagger}$ \\
0.0 & ---     & inside  & 0.7325 & 1.0000 & 0.968500 & 0.000$^{\dagger}$ & 1.769 \\
0.6 & facing  & inside  & 0.6875 & 1.0000 & 0.862969 & 0.000$^{\dagger}$ & 0.663 \\
0.6 & hidden  & inside  & 0.6700 & 1.0000 & 0.854094 & 0.000$^{\dagger}$ & 1.769 \\
1.2 & facing  & inside  & 0.6100 & 1.0000 & 0.774062 & 0.000$^{\dagger}$ & 0.351 \\
1.2 & hidden  & inside  & 0.5975 & 1.0000 & 0.757500 & 0.000$^{\dagger}$ & 1.741 \\
\bottomrule
\end{tabular}
\caption{The full mechanism grid. One wrong-topology artifact:
unfalsifiable and harmless ($\gamma = 0$, the bitwise theorem), falsifiable
and costly (facing channel), instantly falsified (inside start). The hidden
channel is observationally identical to the closed ring although it changes
$\beta_1$. Columns: $r$ contact rarity, $r_{\mathrm{int}}$ interior
entry, $\mathrm{disagree}_{\mathrm{fill}}$ per-transition filled-model
disagreement, $\mathrm{pc}$ \texttt{play\_cost}.
Rates over $400$ rollouts ($r$, $r_{\mathrm{int}}$; the JSON
carries per-cell Wilson CIs, e.g.\ $r = 0.045 \in [0.029, 0.070]$) and
$32{,}000$ transitions ($\mathrm{disagree}_{\mathrm{fill}}$); a zero cell
has 95\% Wilson upper bound $0.0095$ (rollout rates) and $1.2 \times
10^{-4}$ (transition rates) --- every qualitative separation in the table
exceeds its CI by orders of magnitude. $\mathrm{pc}$ over $16$ paired MPC
episodes. $^{\ast}$exact --- $r_{\mathrm{int}}$ and the interior
disagreement by Lemma~\ref{lem:crossing} (no sampled transition can land
inside), $\mathrm{pc}_{\mathrm{fill}}$ bitwise by
Proposition~\ref{prop:harmless}; $^{\dagger}$censored --- no occurrence
in that sample, content is the stated upper bound (episode-rate zeros:
$0/16$, Wilson upper $0.19$).}
\label{tab:mechanism}
\end{table}

Three readings organize the grid. \textbf{First, the three-regime walk is
real.} The same filled-disc artifact is unfalsifiable and costless at
$\gamma = 0$ (disagreement rate $0$, $\mathrm{pc}_{\mathrm{fill}} = 0$
--- the bitwise theorem observed), falsifiable at rate
${\sim}10^{-3}$/transition and costly through a facing channel
($\mathrm{pc}_{\mathrm{fill}}$ $0.343$/$0.220$), and instantly falsified
from inside (disagreement $0.77$--$0.97$). One artifact, three
certification regimes, two knobs.

\textbf{Second, an invented mode is exploited below random from inside.}
$\mathrm{pc}_{\mathrm{fill}} = 1.769$ at $\gamma = 0$, inside start: the
filled disc hallucinates freezes everywhere near the lode the agent already
sits on, so all imagined returns tie and the planner drifts \emph{off} the
reward. This is the dual of the companion papers' lure: an omitted mode
pulls the planner into danger; an invented one repels it from value, just
as destructively.

\textbf{Third --- the thesis, pre-LLM --- policy-relative reachability
beats topology.} The hidden-channel rows are observationally
\emph{identical} to the closed ring at every tolerance (interior entries
$0$, disagreement $0$, $\mathrm{pc}_{\mathrm{blind}}$ $0.999$,
$\mathrm{pc}_{\mathrm{fill}}$ $0.000$) although the channel changes the
free space's connectivity; with the \emph{facing} channel --- same
topology --- everything changes. What certificates and play see is the
mode's position relative to the operative reach. Note the two grades of
impossibility beneath the identical observations: the closed interior is
reach-null \emph{exactly} (a theorem), the hidden channel's interior is
reachable at positive-but-unmeasurable rate; Section~\ref{sec:openring}
returns to this pair with the danger curve.

\section{LLM synthesis on the closed ring: three families}
\label{sec:synthesis}

Throughout the experimental sections an \emph{arm} is a prompt/start
configuration, a \emph{cell} fixes every knob within an arm, and a
\emph{seed block} is the disjoint set of rollout seeds behind one
synthesis draw --- the experimental unit every interval in this paper is
computed over. The synthesis arm runs the companion paper's protocol
verbatim
(Azure GPT-5.x mini and large, $20$ seeds/cell, $N = 40$,
$\varepsilon = 10^{-9}$, $\le 5$ refines) on four pre-registered cells of
the closed ring: \textbf{A} (default prompt, outside start), \textbf{B}
(region-first guidance, outside), \textbf{C} (guidance plus a per-seed
topological summary of the sample's own contact evidence, outside), and
\textbf{D} (the same summary, \emph{inside} start). Because a filled disc
at $\gamma = 0$ passes every metric (Proposition~\ref{prop:harmless}),
code-level claims need more than the gate; the audit below is used for
every artifact in the paper.

\subsection{The behavioral audit (methodology)}
\label{sec:audit}

Textual inspection of synthesized code misleads in both directions: a
verbose ``trap'' comment can be measure-zero (behaviorally blind), and a
plausible-looking region clause can be wildly mis-parameterized. The audit
therefore classifies each artifact by what its \texttt{step()}
\emph{does}: probe it on an $81\times81$ state grid at two velocity
slices, mark every deviation from the pinned integrator (the artifact's
effective freeze/mode set), and classify that set's shape by angular
coverage around the ring center, interior fill, and boundedness:
\texttt{blind} / \texttt{point} / \texttt{vdep} (velocity-conditioned) /
\texttt{arc} / \texttt{loop} (hollow, bounded) / \texttt{disc} (filled) /
\texttt{complement} (hollow, unbounded) / \texttt{fill-unbounded}
(half-plane-like). The classifier is oracle-tested: artifacts constructed
with a known class by construction (a hollow annulus, a filled disc, a
half-plane, a velocity cap, a measure-zero trap, \ldots) must be labeled
correctly before the classifier touches real artifacts.

Two derived readings recur. \textbf{Blind-reference gates}: compare each
terminal gate with the pure integrator's gate on the \emph{same}
evidence (${\approx}0.999$ outside, $0.971$ inside). The empirical gate
bands are: correct ${\approx}1.0$; blind ${\approx}0.97$; posed-but-%
mis-parameterized hollow structures $0.5$--$0.9$; and interior-filling
structures $0.02$--$0.4$, because they freeze states the evidence shows
moving. A posed region with wrong parameters scores far \emph{below}
blind, so a lower mean gate can mean \emph{more} mode-posing, not less
(Section~\ref{sec:openring} shows this reversal live). \textbf{Per-patch
coverage}: the same mask measures whether an artifact encodes a
\emph{specific} region (e.g., the seen versus unseen patch of the companion
paper's bi-modal instrument, where the audit verified the published hand
inspection at $74/76$ integrator exactness and $0$ seen-patch coverage).

\begin{table}[h]
\centering
\footnotesize
\begin{tabular}{llrrrl}
\toprule
cell & prompt / start & pass & present & blind ref & terminal classes (behavioral) \\
\midrule
A$_{\mathrm{lg}}$ & default / out & 6/20 & 14 & 0.999 & fill-unbounded 13, blind 7 \\
A$_{\mathrm{mi}}$ & default / out & 6/20 & 14 & 0.999 & blind 15, arc 2, disc 2, fill-unb.\ 1 \\
B$_{\mathrm{lg}}$ & region / out  & 8/20 & 14 & 0.999 & blind 16, arc 2, disc 1, point 1 \\
C$_{\mathrm{lg}}$ & tda / out     & 9/20 & 14 & 0.999 & blind 19, vdep 1 \\
C$_{\mathrm{mi}}$ & tda / out     & 12/20 & 14 & 0.999 & blind 17, arc 2, point 1 \\
D$_{\mathrm{lg}}$ & tda / inside  & 1/20 & 20 & 0.971 & arc 5, point 4, blind 4, disc 2, compl.\ 2, loop 2, vdep 1 \\
D$_{\mathrm{mi}}$ & tda / inside  & 0/20 & 20 & 0.971 & disc 5, loop 5, arc 5, blind 4, compl.\ 1 \\
\bottomrule
\end{tabular}
\caption{GPT-5.x on the closed ring ($\gamma = 0$; lg = large, mi =
mini). ``tda'' denotes the per-seed topological-data-analysis summary
used as guidance in cells C and D; ``present'' counts mode-present
training samples; ``blind ref'' is the pure integrator's gate on the
same evidence. The two sizes share seeds, so their samples and
mode-present counts are identical by construction, not independent
replications.
\texttt{results/continuous\_synthesis\_ring2d\_\{mini,large\}\_gap0\{,\_pv-tda,-in\_pv-tda\}.json};
unit: seed block.}
\label{tab:closedring}
\end{table}

\paragraph{The identifiability event is family-independent.} Every
mode-absent sample --- probability $(1 - r)^{40} = 0.28$ at the defaults,
with $r = 0.0312$ measured at $30{,}000$ calibration rollouts
(\texttt{results/ring2d\_rarity\_sweep.json}; realized: $6/20$ seed
blocks, the unit both sizes share) --- yielded a certified, fully
mode-blind artifact, exploited at
$\mathrm{play\_cost} \approx 1.12$ with contact rate $1.0$, in every
family tested: GPT-5.x (both sizes, full arm $20/20$ clean in both), Qwen
(Hugging Face router, spot check: both mode-absent seeds certified blind
and exploited at the same $1.116$, the mode-present seed refused by the
gate), and Claude (agent-relayed, spot check: the same A-cell outcome
bit-for-bit). The play
number is bit-identical across families because the blind artifact is the
same program: identifiability is a property of the sample, not the model
--- the companion papers' law, now on a mode with an inside.

\paragraph{Held-out replication: independent-gate acceptance coincides
with that gate's own mode-miss, $156/156$.} An independent-gate audit
re-scores all $903$ synthesized artifacts from the $39$ synthesis
conditions --- the thin-neck cells of
Section~\ref{sec:thinneck} and the two summary-intervention arms of
Section~\ref{sec:sensor} included --- on a disjoint $40$-rollout gate
block per seed
(\texttt{scripts/heldout\_gate\_audit.py}, \texttt{ring2d} scope; output
\texttt{results/heldout\_\allowbreak gate\_\allowbreak
audit\_\allowbreak ring2d.json}), turning the identifiability law into a
held-out
measurement. On the incomplete arm restricted to artifacts whose
\emph{training} sample missed the mode, held-out acceptance coincides
with ``the independent gate's own block also missed the mode'' in
$156/156$ artifacts --- an exact identity checked per artifact, over
$91$ disjoint seed blocks (the two sizes share their blocks by
construction), zero off-diagonal --- and the two-factor prediction
$(1-r)^{N_{\mathrm{train}}+N_{\mathrm{gate}}}$ lies inside
the Wilson $95\%$ interval of the measured accepted-and-blind rate in
every one of the $39$ synthesis conditions (each interval over that
condition's own
disjoint, independent seed blocks --- never pooled across sizes), the
hidden-channel cells included --- audited
under the closed ring's $r$, which their own $30$k calibration equals to
five decimals. The same audit measures the converse: of $214$ in-sample
gate passes, $121$ fail the independent gate, and every one of the $121$
fails \emph{at a mode contact} --- an in-sample pass is training-set
consistency, and what it omits is exactly the mode. Acceptance --- by
the training sample or by an independent block --- is a sample event on
this instrument in both directions, measured on held-out evidence.

\paragraph{None of the three tested families repairs the ring from outside
evidence.} Across A--C,
zero mode-present artifacts pass the gate with the mode genuinely encoded.
The behavioral audit is what makes this claim sharp: the gate-passing
mode-present artifacts in B and C are \emph{textual point fits} ---
integrator plus a comment hypothesizing a tiny localized trap and an
exact-coordinate snap at $10^{-12}$ tolerance --- whose freeze-mask is
empty: behaviorally blind ($\mathrm{wall\_blindness} = 1.0$), exploited
like the pure blinds. The held-out audit supplies the class's sharpest
specimen: one trap (a hidden-channel cell, observationally the closed
ring) freezes on exact \texttt{==} equality with its sample's single
contact state, and the hardcoded coordinate sits $2$ ulps from the same
trajectory under a different \texttt{libm} --- its stored in-sample
$1.000$ is a property of the last bit of $\sin/\cos$ on the machine that
generated it, and the independent gate rejects it on every platform
(\texttt{results/heldout\_gate\_audit\_ring2d.json},
\texttt{train\_reproduction\_check}). This is Corollary~\ref{cor:evidence} operating:
from outside, ring and disc evidence are pathwise identical, so the
honest summary can only report the reachable arc, and it does not pose the
hole. Notably, richer prompting made the \emph{terminal} attempts less
geometric, not more: the default prompt's failing terminals are half-plane
templates ($13/14$ mode-present, large), while region/summary guidance
pushed terminals toward point fits and blinds ($16$--$19/20$).

\paragraph{From inside, repair is parameter-bound, not structure-bound.}
The D cell inverts the failure mode. With inside evidence and the
loop-reporting summary, models \emph{pose} geometric mode structure in
most cells --- mini: $5$ filled discs, $5$ hollow loops, $1$ complement,
$5$ arcs of $20$ --- yet $0/20$ pass (large: $1/20$), and the terminal
gates average $0.56$, far \emph{below} the $0.971$ blind reference: a
posed region with wrong parameters over-freezes states the evidence shows
free. The bottleneck is not posing the topology (the guidance says loop and
the models comply) but pinning its parameters to $10^{-9}$ from landing
evidence. A dose curve makes this precise: re-running the D cell with
$8\times$ the evidence ($N \in \{40, 80, 160, 320\}$ rollouts, $20$
seeds each) yields $0/20$ gate passes at \emph{every} dose, with the
posed structures' median inner-radius error pinned at $0.5$ throughout ---
that is, the median artifact poses $r = 4.0$, a \emph{round number},
whether it has seen $40$ rollouts or $320$. The parameter is not estimated
from the evidence; it is guessed from the template library, so evidence
dose cannot buy repair --- the synthesis-side twin of
Section~\ref{sec:sensor}'s dose-independence.

\paragraph{The separator ablation: a square ring.} Replacing the Euclidean
norm by the Chebyshev norm turns the annulus into a square ring (zero
curvature, corners) while the crossing lemma survives verbatim (Chebyshev
distance is 1-Lipschitz with respect to Euclidean steps; reach-null and the
bitwise Proposition~\ref{prop:harmless} equivalence re-verify on the
square). Every quantitative finding is norm-invariant: mode-absent samples
yield certified blind artifacts exploited at $1.117$; inside repair is
$0/20$ with structures posed in $15/20$. The shape of what is posed is not:
of the $12$ posed closed structures on \emph{square} evidence, $11$ are
round (freeze-boundary corner ratio ${\approx}1.0$ against the square's
$\sqrt 2$) and $12/12$ are written with \texttt{hypot} --- zero use the
truth's \texttt{max}/\texttt{abs} form. The companion paper's
bidirectional template prior, measured there on small patches, reaches the
enclosing separator itself: the topology axis and the template-prior axis
compose, and the gate-certified artifact's \emph{shape} comes from the prior
even when its \emph{topology} comes from the guidance. The D failures
are not template-prior collapse: the behavioral audit classifies them as
right-structure-wrong-parameters, which is a different failure mode.

\paragraph{The sole gate-certified recovery uses the gauge-free
disc-complement form.} The single gate-certified D recovery
(large, $1/20$) freezes exactly when the post-integration landing has
$\lVert p - c \rVert \ge 3.5$: the \emph{disc-complement} form. It is
certified perfect on its own gate sample (gate $1.0$, blindness $0.0$, $\mathrm{play\_cost}$
$0.0$) although its outer boundary is wrong --- from inside, the probe
freezes at the inner boundary and never reaches the outer one, so the
complement is indistinguishable from the annulus on $\mathcal{R}$: the
outer boundary is \textbf{gauge} (Proposition~\ref{prop:quotient}), and
the artifact's comment even cites the summary's ``closed loop
($\beta_1 = 1$).'' Its single load-bearing parameter is the true inner
radius --- a \emph{round number} anchored at the center that the
contract's reward section already names. The strongest cross-family
repairer (Claude, agent-relayed context-free, $3/3$ D repairs) wrote
exactly the same form. Gate-pass from inside therefore measures a
conjunction --- \emph{(reachable-equivalent structure)} $\times$
\emph{(exact-parameter guessability)} --- not topology-understanding
alone; Section~\ref{sec:openring} exploits this by moving the true
parameters off round values.

\section{The open-ring arm: danger is topology relative to reach}
\label{sec:openring}

The registered arm sweeps $\gamma$ with three pre-registered hypotheses:
\textbf{H1}, the blind model's exploitation collapses with a facing
channel and persists with a hidden one (danger is reachability, not
$\beta_1$); \textbf{H2}, posed artifact structure tracks the guidance's
reported topology (Section~\ref{sec:sensor}); \textbf{H3}, gate-pass from
inside stays ${\approx}0$ at every $\gamma$ because the channel-edge
angles $\pi \pm \gamma/2$ are not round-guessable. Cells: outside/default
at $\gamma \in \{0.05, 0.1, 0.2, 0.4, 0.6, 1.2\}$ facing ($20$ seeds
mini) and $\{0.6, 1.2\}$ hidden ($10$); inside/summary at
$\{0.2, 0.6, 1.2, 1.8, 2.4\}$ ($20$--$30$); large robustness cells at the
critical points; a dense CPU danger curve ($16$ paired MPC episodes per
$\gamma$); and Claude relay spots at the knee and at wide $\gamma$.

\begin{figure}[t]
\centering
\includegraphics[width=0.72\linewidth]{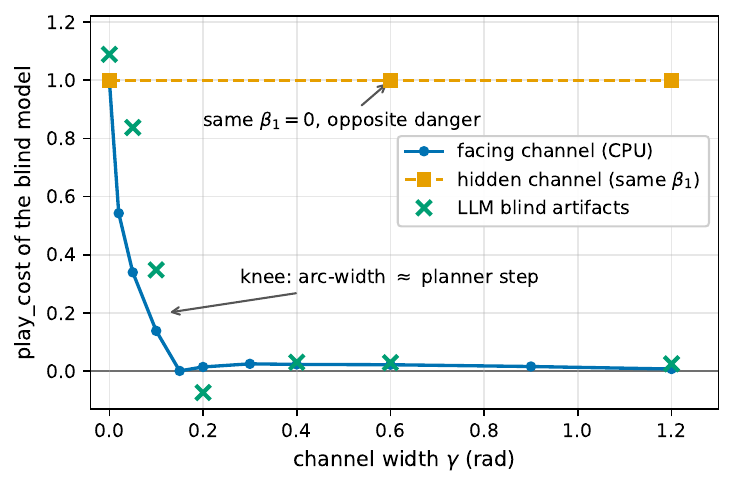}
\caption{Blind-model play cost versus channel width $\gamma$. Blue: the
dense CPU sweep of the facing channel ($16$ paired MPC episodes per
point); green: the LLM arms' exploited blind artifacts (the same blind
program in all three model families); orange: the hidden-channel control
at the same $\gamma$ and the same $\beta_1 = 0$.}
\label{fig:dangercurve}
\end{figure}

\begin{table}[h]
\centering
\small
\begin{tabular}{lrrrrrrrrr}
\toprule
$\gamma$ & 0.0 & 0.02 & 0.05 & 0.1 & 0.15 & 0.2 & 0.4 & 0.6 & 1.2 \\
\midrule
$\mathrm{pc}_{\mathrm{blind}}$ (facing) & 0.999 & 0.543 & 0.340 & 0.139 & 0.001 & 0.015 & 0.023 & 0.022 & 0.007 \\
contact rate & 1.00 & 0.94 & 0.56 & 0.44 & 0.12 & 0.06 & 0.00$^{\dagger}$ & 0.00$^{\dagger}$ & 0.00$^{\dagger}$ \\
\bottomrule
\end{tabular}
\caption{The facing danger curve (CPU, $16$ paired MPC episodes per
point; $\mathrm{pc}$ denotes \texttt{play\_cost};
\texttt{results/continuous\_ring2d\_open\_sweep\_summary.json}, one
seed-block cell per $(\gamma, \text{arm}, \text{size})$). The
\emph{hidden} channel holds $\mathrm{pc} = 1.116$ at $\gamma = 0.6$
\emph{and} $1.2$. Figure~\ref{fig:dangercurve} plots all three series.
$^{\dagger}$censored: $0/16$ episodes, Wilson upper $0.19$.}
\label{tab:dangercurve}
\end{table}

\paragraph{H1 confirmed: a facing channel collapses danger; an equally
wide hidden channel keeps it.} The facing
curve collapses
over a knee at $\gamma \approx 0.1$--$0.15$; the LLM cells reproduce it on
their exploited blind artifacts in both sizes and in the Claude relay
($0.348$ at the knee in all three families --- the play number is the blind
program's, family-independent by construction). The \emph{hidden} channel
at the same $\gamma$ holds $\mathrm{play\_cost} = 1.116$ --- full
strength, unchanged from the closed ring. Same $\beta_1 = 0$, opposite
danger: what collapses the danger is the competent planner's reachability
of the omission, not the topology change. The mechanism is the gate
quotient on the play side: as the channel opens where the planner drives,
the phantom stops being phantom --- the blind plan (straight at the lode)
becomes \emph{executable in the truth}, so blind and truth agree along the
operative path exactly as $E(f)$-members agree on $\mathcal{R}$. The knee
sits where the channel's arc width
($\gamma\, r_{\mathrm{in}} \approx 0.35$ at $\gamma = 0.1$) admits the
planner's step.

\paragraph{H3 confirmed --- certified passes are not repairs.} Inside
gate-pass rates across $\gamma$: $1/40$, $0/20$, $0/40$, $0/20$, $1/30$,
$1/40$ --- essentially zero everywhere, both sizes, with terminal gates far
below the blind reference (structures posed, parameters unpinnable: the
channel edges are not round numbers). The sharpest datum is the strongest
repairer's: at $\gamma = 2.4$ the Claude relay \emph{reconstructs} ``ring
$r = 3.5$, gap $2.4$\,rad'' from the sample and passes the gate in $2/3$
seeds --- with mode blindness $1.0$ and $\mathrm{play\_cost}$ $0.0$. The
passing models match the sample through wrong parameters, fail the mode
probes, and are harmless only because the wide-open ring no longer
obstructs anything. Certification, correctness, and consequence come apart
again (Remark~\ref{rem:triad}), now with real synthesis. Held-out
re-scoring adds what the inside failures are made of: $237/320$
inside-start artifacts fail the independent evaluation \emph{off} the
mode as well (accuracies down to $0.05$), against $87/343$ for outside
starts (\texttt{results/heldout\_gate\_audit\_ring2d.json}, unit
artifact) --- the loop-evidence regime produces globally wrong models,
not clean integrators missing a clause.

\paragraph{The identifiability axis moves independently alongside.} As
$\gamma$ grows the ring shrinks and the mode-absent rate rises ($6 \to 16$
of $20$ over the sweep): the synthesis axis (does the sample contain the
mode?) and the play axis (does the omission obstruct the planner?) are
separate knobs even within one instrument --- the theme
Section~\ref{sec:ndim} pushes to its extreme.

\subsection{The metric converse: a thin neck, topology held fixed}
\label{sec:thinneck}

The $\gamma$-channel reopens the interior \emph{topologically}. Its
converse thins the band from outside inside a $0.3$-rad sector
($r_{\mathrm{out}}$ dips to $r_{\mathrm{in}} + \mathrm{neck}$; the hole is
invariant): $\beta_1 = 1$ and the band still separates the plane in the
continuum at every neck $> 0$, but interior entry now requires a
\emph{single step longer than the neck} --- the local crossing lemma,
machine-checked in Lean beside Lemma~\ref{lem:crossing}: a mode set
containing the thin annulus $[r_{\mathrm{in}}, r_{\mathrm{in}} + w]$ with
$w > \Delta$ seals the hole whatever its shape. With $\Delta = 1.0$ the
knob crosses Lemma~\ref{lem:crossing}'s metric hypothesis while the
topology never moves, and reachability below the threshold is
\emph{exhibited}, not sampled: a deterministic constant-thrust witness
leaps the $0.5$ neck (one $0.547$-step from $d = 4.03$), and the same
$40$-witness family is blocked at neck $\ge 1.0$.

\begin{table}[h]
\centering
\small
\begin{tabular}{lrrrrr}
\toprule
neck (facing) & $r$ & $r_{\mathrm{int}}$ & $\mathrm{dis}_{\mathrm{fill}}$ &
$\mathrm{pc}_{\mathrm{blind}}$ & $\mathrm{pc}_{\mathrm{fill}}$ \\
\midrule
closed & 0.0312 & $0^{\ast}$ & $0^{\ast}$ & 0.999 & $0.000^{\ast}$ \\
0.1 & 0.0227 & 0.0020 & $4.3\times10^{-4}$ & 0.451 & 0.594 \\
0.2 & 0.0248 & $1.7\times10^{-4}$ & $0^{\dagger}$ & 0.962 & 0.573 \\
0.4 & 0.0255 & $0^{\dagger}$ & $0^{\dagger}$ & 0.971 & 0.500 \\
0.6 & 0.0262 & $0^{\dagger}$ & $0^{\dagger}$ & 0.962 & $0.000^{\dagger}$ \\
0.8 & 0.0269 & $0^{\dagger}$ & $0^{\dagger}$ & 0.976 & $0.000^{\dagger}$ \\
1.2 & 0.0292 & $0^{\ast}$ & $0^{\ast}$ & 0.993 & $0.000^{\ast}$ \\
\bottomrule
\end{tabular}
\caption{The thin-neck control, topology held fixed ($r$ contact rarity,
$r_{\mathrm{int}}$ interior entry, $\mathrm{dis}_{\mathrm{fill}}$
per-transition filled-model disagreement, $\mathrm{pc}$
\texttt{play\_cost}). Each cell uses $30{,}000$ gate rollouts,
$320{,}000$ disagreement transitions, and $16$ paired MPC episodes
(\texttt{results/ring2d\_thin\_neck.json}; design and readings
pre-registered before the run). $^{\ast}$exact --- the
closed row by Lemma~\ref{lem:crossing} and
Proposition~\ref{prop:harmless}, the neck-$1.2$ row by the local lemma at
$\Delta = 1.0 < 1.2$, imagined steps included. $^{\dagger}$censored ---
$0$ occurrences in that sample; Wilson $95\%$ uppers $1.3\times10^{-4}$
(rollouts), $1.2\times10^{-5}$ (transitions), $0.19$ (episodes;
\texttt{results/ring2d\_zero\_wilson.json}, emitted from the sweep's own
denominators).}
\label{tab:thinneck}
\end{table}

Table~\ref{tab:thinneck} carries three findings. \textbf{Hidden necks
are bit-identical to the closed
ring} at all six thicknesses (the same $r$ to the last float, $0$
entries, $\mathrm{pc}_{\mathrm{blind}} = 0.999$): reach, not geometry, a
third time on this instrument. \textbf{The planner leaks only at the
thinnest neck}: $\mathrm{pc}_{\mathrm{blind}}$ collapses to $0.451$ at
neck $0.1$ (blind contact rate $0.56$ --- the blind planner escapes
through the neck in half the paired episodes) and recovers to
$0.96$--$0.99$ from neck $0.2$ up, because one landing in the thin band
pins the episode (freeze zeroes the velocity, and from rest the next
step is $R_L = 0.03$): threading needs the landing phase to skip a band
the realized approach rarely skips. \textbf{Certified-and-costly, on the
metric axis}: the filled-disc model can disagree with truth on a sampled
transition only via a leap --- $139/320{,}000$ transitions at neck
$0.1$, $0/320{,}000$ at every thicker neck --- yet
$\mathrm{pc}_{\mathrm{fill}} = 0.59/0.57/0.50$ at neck $0.1/0.2/0.4$,
with fill contact rate $0.00$: the cost is plan divergence, because the
\emph{planner's imagined candidates} leap (imagination reaches speed
$> 4$) where the random gate's rollouts never do (arrival speeds
$\approx 1.5$). At neck $\in \{0.2, 0.4\}$ the wrong topology is
behaviorally indistinguishable from truth across $320{,}000$ sampled
transitions and $30{,}000$ rollouts while costing half the play margin
--- the companion paper's certified-region/query-mass gap, reproduced
with the topology knob held fixed. The gate certifies where it looks;
the planner leaps where it plans.

\paragraph{The loop does not write the neck: $0/120$ artifacts pose any
angular structure.} The synthesis analogue of posing the hole is posing
the dip, and at neck $\in \{0.1, 0.2, 0.4\}$ (facing, incomplete arm,
both sizes, $20$ seeds per cell;
\texttt{results/continuous\_\allowbreak synthesis\_\allowbreak
ring2d\_\allowbreak \{mini,large\}\_\allowbreak gap0-nk\{0.1,0.2,0.4\}.json},
scanned by \texttt{results/ring2d\_\allowbreak neck\_\allowbreak
synthesis\_\allowbreak scan.json}) no artifact does either: $0/120$
contain an angular term of any kind and $0/120$ pose even a
\emph{uniform} band --- the conditionals that appear are one-sided discs
at the reward radius, velocity thresholds, and three exact-coordinate
point fits of the family above, one of which (freezing on \texttt{==}
equality with its training sample's single contact state) reached
in-sample $1.000$ on the platform that synthesized it and fails the
independent gate \texttt{mode\_only}. Evidence is not the bottleneck:
the two seed blocks at neck $0.1$ whose training samples contain
leap-through transitions ($16$ and $12$ interior landings; the sizes
share their blocks by construction) are answered by a reward-threshold
freeze and by free flight --- shown the metric hole, the loop writes no
geometry at all. All $45$ gate passes across the six cells are
wall-blind, and the held-out audit absorbs the cells without strain
(the $39/39$ and $156/156$ above). The regime above --- certified
relative to the gate's sampled transitions, costly in play --- therefore
carries a synthesis-side converse: at necks the gate cannot distinguish
from the closed ring, the accepted artifacts are the same blind
integrators as everywhere else, admitted by the same two-factor law.

\section{Mitigation against an enclosed mode: a covering law}
\label{sec:mitigation}

The companion continuous paper closed its loop with a planner-side defense:
\emph{distrust-region replanning}. After each real step the planner compares
the model's prediction to the observation; a mismatch records a point
\emph{fence} at the refuted prediction, and imagined rollouts truncate when
they cross near a fence --- on the 1D clamps and the small convex patches
this collapses the blind planner's exploitation to a bounded first-contact
transient, at zero cost when the model is correct. The same module, verbatim
(\texttt{pos\_dims} $=(0,1)$, the patch-calibrated fence radius
$\varepsilon_f = 0.5$), meets the ring:

\begin{table}[h]
\centering
\small
\begin{tabular}{llrrrr}
\toprule
cell & $\varepsilon_f$ & $\mathrm{pc}_{\mathrm{model}}$ &
$\mathrm{pc}_{\mathrm{mitigated}}$ & contact & fences/episode \\
\midrule
closed ring, outside, blind & 0.5 & 0.999 & 1.003 & 1.00 & 3.8 \\
closed ring, outside, blind & 1.0 & 0.999 & 1.005 & 1.00 & 2.3 \\
closed ring, outside, blind & 2.0 & 0.999 & 0.742 & 1.00 & 1.0 \\
hidden $\gamma = 0.6$, outside, blind & 0.5 & 0.999 & 1.003 & 1.00 & 3.8 \\
closed ring, \emph{inside, filled} & 0.5 & 1.769 & \textbf{1.769} & 1.00 & 18.6 \\
\midrule
nerve fence, episodic & 0.5 & 0.999 & 0.957 & 1.00 & 2.0 \\
\textbf{nerve fence, persistent} & 0.5 & 0.999 & \textbf{0.058} & 0.06 & 0.1 \\
\textbf{freedom patch} (inside, filled) & 0.5 & 1.769 & \textbf{0.029} & --- & 80 \\
\bottomrule
\end{tabular}
\caption{Planner-side mitigation on the ring ($16$ paired MPC episodes
per cell;
\texttt{results/continuous\_mitigation\_ring\{,\_eps1,\_eps2,\_nerve\}.json}
and \texttt{results/continuous\_ring2d\_optimism.json}; unit: paired
episode). $\mathrm{pc}_{\mathrm{model}}$ and
$\mathrm{pc}_{\mathrm{mitigated}}$ are \texttt{play\_cost} before and
under the defense. A \emph{nerve fence} links violations into
tangentially extended segments; a \emph{freedom patch} locally
unfreezes imagination where the model was refuted as too pessimistic.
Rows 1--5 use point fences, rows 6--7 compare episodic and persistent
nerve fences, and the final row applies freedom patching to the
inside-start filled model.}
\label{tab:mitigation}
\end{table}

\paragraph{The failure is a covering law.} A point fence is a
zero-dimensional object; the ring's reachable boundary is one-dimensional
(the west outer arc, ${\sim}16$ world-units at the defaults). Sealing a
1-dimensional boundary with $\varepsilon_f$-balls needs
covering-number-many violations --- boundary measure over fence radius ---
and the argmax planner concedes only $2$--$4$ contacts per episode while
re-routing through unfenced arc. The companion paper's ``collapse decays
with mode distance'' was the small-boundary shadow of this law: on a patch
of circumference ${\sim}6$ at $\varepsilon_f = 0.5$ a handful of fences
suffices; on the ring they never do. Coarsening the fence to the geometry's
own scale ($\varepsilon_f = 2.0$) buys partial relief here but is not a
uniform remedy --- that radius would swallow the patch instruments whole.
The structural fix is dimensional --- and it works, but only jointly with
persistence. A \emph{nerve} fence links violations into segments and
extends them $3$ units along the estimated boundary tangent (two violations
$0.5$ apart already fix the local direction). Episodically it still fails
($\mathrm{pc}$ $0.957$): a single episode concedes only ${\sim}2$ lessons,
and instrumentation shows the planner simply re-crosses the wall
\emph{in imagination} beyond the sealed corridor --- the blind model lets
any unfenced crossing through, so the agent slides along the boundary at
one freeze per lesson. Persisting the fences across episodes (deployment
monitoring rather than per-episode defense) changes the accounting: the
boundary's covering cost is paid \emph{once}. Measured: $2$ violations in
episode~1 (return $1.01$), then \textbf{zero} violations and returns equal
to the truth planner's in every one of episodes $2$--$16$ ($17.04$ vs
$17.04$, \ldots); aggregate $\mathrm{pc}$ $0.999 \to 0.058$, contact
rate $1/16$. The mechanism is Proposition~\ref{prop:quotient} run in
reverse: the extended fence reconstructs enough of the boundary that the
blind model \emph{plus its fence memory} becomes planner-equivalent to the
truth on the operative side --- certification could not distinguish members
of $E(f)$, and the defense engineers the planner back into it. This is
theorem-backed under checkable hypotheses:

\begin{proposition}[fence sufficiency]
\label{prop:fence}
Suppose \textup{(COV)} every segment of length $\le \Delta$ crossing the
reachable outer boundary arc of $A$ passes within $\varepsilon_f$ of the
fence set, and \textup{(RG-west)} the best non-crossing candidate's
imagined return exceeds every crossing candidate's pre-crossing return plus
its tie-break advantage (at the defaults: ${\approx}2$ versus
${\le}0.1$, checkable over the visited envelope). Then every crossing
candidate truncates at or before its crossing; every non-crossing candidate
scores identically under truth imagination and fenced-blind imagination
(the models agree off the band, and truth imagination freezes only on
crossings); and the argmax coincides. The mitigated planner \emph{is} the
truth planner: $\mathrm{play\_cost} = 0$ exactly.
\end{proposition}

Episode 1 is the learning transient the proposition does not (and should
not) cover: (COV) is established \emph{by} its two lessons. And (COV) is
policy-payable by three routes, measured head-to-head (all persistent):
fence \emph{geometry} (the tangential extension: $2$ lessons,
$\mathrm{pc}$ $0.058$), \emph{exploration} (a boundary-tracing probe
collects $50$ lessons in one episode --- plus $11$ patch-ups, since its
$1.05$-unit spacing left sub-$\varepsilon_f$ gaps: the covering law bites
the prober too --- $\mathrm{pc}$ $0.123$), and \emph{passivity} (the
planner's own contacts, $9$ lessons spread over $16$ episodes,
$\mathrm{pc}$ $0.152$ with occasional regressions). Persistence is the
ingredient common to everything that works; geometry beats exploration
beats waiting. The fence is itself an estimated boundary ---
Section~\ref{sec:sensor} shows next that any such boundary
representation inherits the evidence sensor's finite-resolution limit.

\paragraph{Against an invented mode, distrust is inert by construction.}
The filled-disc model from inside (the below-random cell of
Section~\ref{sec:mechanism}) hallucinates freezes everywhere, so every
imagined future collapses to the current position and all candidate actions
tie \emph{before} any fence can matter; truncating trust in imagined value
cannot help when the lie has already flattened imagined value. The fences
fire constantly ($18.6$ per episode --- the defense \emph{detects} the lie
every few steps) and change nothing ($\mathrm{pc}$ identical to three
decimals). Distrust-based mitigation is one-sided in exactly the sense the
danger is two-sided: it can stop the planner from \emph{believing} false
free space, but recovering value that the model hides behind false
obstructions needs the \emph{dual} certificate. We built it:
\textbf{freedom patching} records, at each real step where the model
predicted a freeze and the truth moved, a freedom point; during
imagination, a model step that freezes within $\varepsilon_f$ of a freedom
point is replaced by the contract's own pinned integrator (mode-free,
legitimately known to the pipeline). On the same cell this collapses the
below-random exploitation at once: $\mathrm{pc}$ $1.769 \to 0.029$
episodic ($0.021$ persistent), with near-truth returns from
\emph{episode 1}. The asymmetry of the two defenses' costs is itself the
finding: an invented mode \emph{lies at every step} ($80$ freedom
certificates per episode --- it refutes itself constantly, so within-episode
patching suffices), while an omitted mode lies \emph{rarely} ($2$ lessons
per episode --- so its fence must persist to ever pay the covering cost).
The dual result is theorem-backed the same way:

\begin{proposition}[patch sufficiency]
\label{prop:patch}
Let $\hat f$ freeze on a region $B$ where $f$ moves freely, and suppose
\textup{(CERT)}: every point of $B$ visited by any candidate's imagined
path from the current real state lies within $\varepsilon_f$ of the
freedom-certificate set. Then patched imagination equals truth imagination
for every candidate (off $B$ the models agree; on the part of $B$ the
freedom certificates cover, the patch substitutes the pinned integrator, which \emph{is} $f$ there),
so the argmax coincides with the truth planner's and
$\mathrm{play\_cost} = 0$ exactly.
\end{proposition}

From inside, every step yields a certificate (the model lies each step ---
the lie rate is the occupation of the disagreement region under the
operative policy), so (CERT) holds within one episode's prefix; the
measured $0.029$ is that prefix. The omitted-mode lure and the invented-mode repulsion
(Remark~\ref{rem:triad}'s two wrongnesses) need \emph{opposite}
planner-side defenses; a defense calibrated for one is inert against the
other; and each defense's price is set by its failure's \emph{lie rate}.

\section{The evidence sensor has finite resolution}
\label{sec:sensor}

The C and D cells' guidance includes a pre-registered
\emph{topological summary} of the sample's own contact evidence: an
honest, shape-agnostic text (never naming a shape family, wording frozen
before any run) reporting cluster count, bounding box, and a persistent
$\hat\beta_1$ from a Rips complex on the deduplicated landing cloud
(grid $0.05$, cap $90$ points, bars above $3\times$ the median
nearest-neighbor spacing). The summary is a \emph{sensor}, and sensors
have resolution: measured on regenerated inside evidence, the detector
reports $\hat\beta_1 = 1$ for every $\gamma \le 1.2$ --- it bridges any
channel narrower than about two arc-units at this density --- flips
seed-dependently at $\gamma = 1.8$, and is cleanly $0$ at $2.4$; the true
$\beta_1$ is $0$ for every $\gamma > 0$. The sweep therefore supplies
the contrasts H2 needs: gaps where
guidance and truth disagree, one gap where guidance disagrees with itself
across seeds, and gaps where they agree.

\paragraph{The resolution limit is geometric, not budgetary.} A factorial
over the detector's budget and the evidence dose
(\texttt{scripts/ring2d\_sensor\_resolution.py}: cap
$\in \{30, 90, 270\}$ points $\times$ evidence
$N \in \{40, 160\}$ rollouts $\times$ $\gamma$, five seeds each)
shows the flip does not move: $\hat\beta_1 = 1$ at every
$\gamma \le 1.2$ under \emph{every} budget and dose, $0$ at $2.4$ under
all of them, with only the boundary gap $1.8$ wobbling.

\paragraph{A two-sided barcode sandwich explains the flip.} The
mechanism is Vietoris--Rips geometry. To the resolution of the proved
sandwich, the detector reports $\hat\beta_1 = 1$ when
$\sqrt{3}\rho - 2\rho \sin(\Delta\theta_{\max}/2) > \tau$, where $\rho$
is the cloud's mean radius, $\Delta\theta_{\max}$ the \emph{largest
angular gap of the sample} (the channel or a subsampling gap, whichever
is larger), and $\tau$ the detector's persistence threshold; this
mean-radius display is the sandwich's collapse at
$r_{\min} \approx r_{\max} \approx \rho$, and as a pointwise predictor
it is measured, not exact. The exact result is two-sided. The winding
bar is born exactly at the largest gap's chord (proved: a winding cycle
must contain an edge spanning every angular gap); it cannot die before
$\sqrt{3}\, r_{\min}$ (proved self-containedly: below that scale every
triangle's boundary has winding zero, so winding obstructs filling); and
it dies by $2 r_{\max} \sin(\theta^*/2)$ with an explicit $\theta^*$
(the filling lemma, proved: an explicit gap-aware spanning triangle of
winding one plus a fan retraction exhibit the 2-chain that bounds the
born \emph{class}, no general filling machinery needed). A pairing lemma
transfers the sandwich from the class to the \emph{bar} whenever exactly
one bar spans
$[2 r_{\max}\sin(\Delta\theta_{\max}/2),\, \sqrt3\,r_{\min})$ --- a rank
condition read off the barcode, true in every factorial row.
Consequently $\hat\beta_1 \ge 1$ is
guaranteed when $\sqrt3\,r_{\min} - 2 r_{\max}\sin(\Delta\theta_{\max}/2)
> \tau$, and no persistent winding bar exists when
$2 r_{\max}\sin(\theta^*/2) - 2 r_{\min}\sin(\Delta\theta_{\max}/2) <
\tau$ (only a spurious non-winding class could still fire --- never
observed). For dense samples $\theta^* \to 2\pi/3$, recovering the
Vietoris--Rips circle constant \citep{adamaszek2017vietoris}, and the
measured death$/\bar\rho \in [1.70, 1.82]$ around $\sqrt 3 = 1.732$ is
the radial thickness. On the factorial: the sandwich holds for $57/57$
winding bars,
the guaranteed bands have $0$ violations ($34$ and $5$ rows), and the
pointwise law reproduces $78/80$ rows, the $2$ misses falling in the
proved undecided band ($\gamma = 1.8$ boundary, subsampling-gap
lottery).

\paragraph{More evidence can strengthen the wrong topological report.}
At $\gamma = 1.8$ the false-loop rate \emph{rises} from $1/5$
to $3/5$ seeds when the dose quadruples ($n = 5$ paired seeds ---
directional, underpowered;
\texttt{results/ring2d\_sensor\_resolution.json}, cap $90$). The per-seed
mechanism is that the denser subsample fills the channel-adjacent shells:
on the two flipping seeds the spurious bar's persistence grows with the
dose ($0.05 \to 0.50$ and $0.11 \to 0.43$), outrunning the threshold,
which itself grows modestly ($4/5$ seeds) --- the sensor gets
\emph{more confident} in the wrong topology as the data grows. Resolving a
narrow channel requires a different filtration, not a bigger sample.

\paragraph{Relative homology avoids the edge-deletion pathology.} A
prototype confirms the direction and locates the real problem: censoring
Rips edges that free trajectory segments properly cross restores
specificity at every $\gamma > 0$ ($19/20$ versus $0/15$ for plain Rips at
$\gamma \le 1.2$) at a sensitivity cost at $\gamma = 0$ ($2/5$ true
loops lost). But edge deletion has a topological pathology with a clean
characterization: a censored filtration stabilizes at the flag complex
of the censor-complement graph, so \emph{never-fillable cycles}
(infinite $H_1$ bars, impossible for plain Rips) are exactly that
complex's $H_1$ --- per-sample decidable by one finite computation.
Running that certificate on all cells finds the pathology is real and
\emph{semantically ambivalent}: the same mechanism carries the true
$\gamma = 0$ loops (their fills cross the certified-free hole, so the
correct loop gets infinite confidence) and the prototype's single false
positive at $\gamma = 0.6$, which turns out to be a structural
never-fillable cycle, not a near-threshold bar.

Edge deletion is therefore the wrong primitive, and the principled
replacement is \emph{relative} homology of the pair
$(\mathrm{VR}(X \cup Y), \mathrm{VR}(Y))$, contact evidence $X$ against
certified-free evidence $Y$. Two things make it the right object. It
cannot produce infinite bars at all --- both complexes are full
simplices past the diameter, so every relative class dies --- and by the
long exact sequence its rank is $\operatorname{rank}\ker(H_0(\mathrm{VR}(Y))
\to H_0(\mathrm{VR}(X \cup Y)))$ whenever the union has trivial $H_1$:
\emph{the number of certified-free components that the contact evidence
glues together}, computable by two union-finds rather than a matrix
reduction. One instantiation detail is load-bearing: $Y$ must enter as
trajectory \emph{polylines}, not as a point cloud. A free trajectory certifies
passage between its own consecutive samples; discard that and the
estimator reports separation almost everywhere, because a curve-dense
contact cloud glues a scatter-sparse free cloud by density mismatch
alone. With paths, the estimator is correct on $20/20$ open-ring cells
where the pre-registered detector scores $9/20$, with no censoring and
no infinite bars. At $\gamma = 0$ it does not fire, for two reasons we
keep separate: with one-sided (inside) evidence the interior is
reach-null, nothing outside is sampled, and separation is
\emph{unidentifiable by any estimator} --- Proposition~\ref{prop:quotient}
asserting itself, with plain $\hat\beta_1$ scoring well there only
because it answers a different question (the cloud's shape, not the
separation it induces). With two-sided evidence the obstruction is
different and, we think, more interesting: freeze semantics resets the
velocity on contact, so after the first freeze the next proposed landing
moves only $\mathrm{gain}\cdot dt^2 = 0.03$ from rest and the trajectory
\emph{creeps along the face} instead of penetrating. Contact evidence
therefore occupies two thin shells at the band's faces and never its
interior, which caps the informative bar at the shell thickness while
the persistence threshold grows with sample size as the shells fill in.
Measured at $40/120/320$ rollouts per arm: shells saturate at
$[3.50, 3.71]$ and $[4.83, 5.00]$, bar length $0.160/0.197/0.253$
against $\tau = 0.183/0.324/0.439$ --- sub-threshold at every size, with
the gap \emph{widening}. So the registered $3\times$ calibration is not
rescued by dose --- the bar-to-threshold gap widens as the data grows;
whether some other single calibration could recover $\gamma = 0$ while
keeping the open-ring specificity is unmeasured.
The same one-step landing law that makes the continuity modulus explicit
(Proposition~\ref{prop:continuity}) is what caps the sensor here. Only
the interleaving-stability theorem, which the pair formulation makes
routine, remains.

\paragraph{Nested boundaries expose a separate resolution failure.}
With two concentric rings (the
multi-chamber instrument) the middle-chamber evidence cloud is two
concentric circles ($\beta_1 = 2$), and the pre-registered detector
returns a seed-dependent \emph{lottery} --- $\hat\beta_1 = 0$ on $9/20$
samples, $1$ on $10/20$, the true $2$ on $1/20$ --- as Rips bridges between
the circles merge or kill loops at this density. The synthesis side is
correspondingly worse than the single ring: $0/20$ gate passes and
\emph{zero} artifacts pose a nested (two-band) structure --- $13/20$ pose
single arcs hugging one boundary. Certification peels one boundary layer
per start placement (Proposition~\ref{prop:quotient}); the repair loop, at
this depth, does not even \emph{pose} layer two.

\begin{table}[h]
\centering
\small
\begin{tabular}{lrrrrrr}
\toprule
$\gamma$ & 0.0 & 0.2 & 0.6 & 1.2 & 1.8 & 2.4 \\
\midrule
closed structures (disc/loop/complement) & 17 & 9 & 22 & 7 & 5 & 1 \\
arc structures & 10 & 5 & 6 & 9 & 34 & 26 \\
$n$ (cells; both sizes where run) & 40 & 20 & 40 & 20 & 60 & 40 \\
guidance $\hat\beta_1$ & 1 & 1 & 1 & 1 & mixed & 0 \\
\bottomrule
\end{tabular}
\caption{H2: posed artifact topology tracks the \emph{guidance's}
$\hat\beta_1$, not the true $\beta_1$ (which is $0$ for every
$\gamma > 0$): closed structures dominate while the summary says ``closed
loop'' and all but vanish where it honestly says ``arc'' --- at
$\gamma = 2.4$, $1$ closed versus $26$ arc across both sizes. $n$ counts inside-start cells
(mini at every $\gamma$; large twins at $0.0$, $0.6$, $2.4$). The
within-gap contrast at $\gamma = 1.8$, where seeds' own summaries straddle
the detector flip, remains directional-only after doubling the cell to
$n = 60$: closed structures $3/27$ in the $\hat\beta_1 = 1$ subset versus
$2/33$ in the $\hat\beta_1 = 0$ subset --- consistent in sign, not
significant; the load-bearing H2 evidence is the cross-gap crossover.}
\label{tab:sensor}
\end{table}

\paragraph{H2: the artifacts follow the sensor.} Closed structures
dominate while the summary says ``closed loop'' and all but vanish
where it honestly says ``arc'' (Table~\ref{tab:sensor}: $1$ closed versus
$26$ arc at $\gamma = 2.4$, both sizes); the crossover tracks the detector
flip, not the true $\beta_1$. The within-gap contrast at $\gamma = 1.8$
--- seeds of the \emph{same} environment whose own summaries straddled the
flip --- is directional only (closed structures $3/27$ in the
$\hat\beta_1 = 1$ subset versus $2/33$ in the $\hat\beta_1 = 0$ subset,
$n = 60$: consistent in sign, not significant). What these measurements
establish is: \textbf{posed topology tracks the sensor's report, not the
truth} --- the cross-gap crossover above. A pre-registered paired
intervention isolates whether the summary's \emph{claim line alone}
causes this association
(\texttt{docs/paper3/INTERVENTION-DESIGN.md}; the analysis was fixed
before any outcome existed): a full crossover at $\gamma = 1.8$ inside
(mini, the same $60$ seeds, bit-identical evidence blocks) whose
treatment prompt negates exactly the summary's $\hat\beta_1$ line and its
one interpretive sentence --- every other prompt byte held fixed ---
scored against a \emph{contemporaneous} honest replicate as primary
control. The paired contrast is directionally consistent with steering
--- of $11$ discordant pairs, $9$ moved toward the claimed topology and
$2$ against --- but does not reach the pre-registered level (exact
two-sided binomial $p = 0.065$ on $60$ pairs; the registered exact
$95\%$ interval for the toward-claim share of discordant pairs is
$[0.48, 0.98]$, spanning $1/2$;
\texttt{results/ring2d\_summary\_intervention.json}), so the causal
reading is \emph{not} earned and H2's conclusion stands as the
association above. Per-seed classification is itself stochastic at this
cell --- the two honest controls (honest-then vs.\ honest-now) disagree
on $9$ of $60$ pairs, split $4{:}5$ ($p = 1.0$): no \emph{observed}
directional imbalance, though the comparison cannot rule out
generation variability as such --- and this seedless generation noise is
exactly what the paired binomial's null absorbs. Gate passage never
moves under the flip ($0$ discordant pairs).
The repair loop's topology channel is, at minimum, no better than the
sensor's resolving power at the operative sampling density.

\paragraph{The same law governs the sensor directly in $n$ dimensions.} On
ShellField-$n$ (Section~\ref{sec:ndim}), persistent homology of the
\emph{outside} contact cloud recovers $\beta_{n-1}$ at \emph{no}
$n \in \{2,\ldots,6\}$ --- at low $n$ the cloud is plentiful but traces only
the reachable arc (signal clarity), at high $n$ it starves as $r(n)$
collapses ($0.18\times$ down to $0.002\times$ the
Niyogi--Smale--Weinberger (NSW) covering density
\citep{niyogi2008finding},
computed on the points the detector actually receives) --- while the
\emph{inside} cloud recovers it cleanly at $n = 2,\ldots,5$. The density
bookkeeping matters: inside clouds \emph{offer} $10$--$2700\times$ the
NSW floor in raw trajectory contacts, but the detector consumes a
deduplicated subsample capped at $300$ points --- $24\times$ the floor at
$n = 2$ falling to $0.22\times$ at $n = 5$ --- so recovery held at and
even below the naive floor (the NSW bound assumes i.i.d.\ uniform
samples; a well-spread deduplicated subsample beats it), and at $n = 6$
only a 2-plane projection diagnostic was computed, which detects
projected circles and is \emph{not} evidence about $\beta_5$. Recovery is
start-governed: reachability again, now for the sensor itself. We flag the
program-level statement this suggests, as a conjecture and not a result:
\emph{the topology a certified repair loop can recover is bounded by the
persistent-homology resolution of its contact evidence at the operative
density} --- a bound that NSW-style estimates make computable in advance,
provided the operative density is the consumed subsample's, not the raw
stream's.

\begin{table}[h]
\centering
\small
\begin{tabular}{lrrlrrl}
\toprule
 & \multicolumn{3}{c}{outside start} & \multicolumn{3}{c}{inside start} \\
\cmidrule(lr){2-4}\cmidrule(lr){5-7}
$n$ & contacts & used ($\times$NSW) & $\beta_{n-1}$? & contacts & used ($\times$NSW) & $\beta_{n-1}$? \\
\midrule
2 & 388 & 219 (17.4)  & \ding{55} & 33{,}702 & 300 (23.8) & \checkmark \\
3 & 167 & 160 (2.5)   & \ding{55} & 45{,}697 & 300 (4.7)  & \checkmark \\
4 & 59  & 55 (0.18)   & \ding{55} & 54{,}170 & 300 (1.0)  & \checkmark \\
5 & 11  & 11 (0.008)  & \ding{55} & 59{,}961 & 300 (0.22) & \checkmark \\
6 & 18  & 15 (0.002)  & \ding{55} & 63{,}949 & 300 (0.05) & proj.\ only \\
\bottomrule
\end{tabular}
\caption{Persistent homology of the ShellField-$n$ contact cloud
($20{,}000$-rollout budget; alpha complexes;
\texttt{results/continuous\_shellfield\_tda\{,\_inside\}.json}, one cell
per $(n, \text{start})$). ``contacts'' counts raw trajectory contacts
(dependent samples); ``used'' is the deduplicated subsample the detector
receives (cap $300$), and ``$\times$NSW'' is \emph{used} relative to the
Niyogi--Smale--Weinberger i.i.d.\ covering bound. Outside evidence never
recovers the shell's $\beta_{n-1}$ --- at low $n$ the cloud traces only
the reachable arc; at high $n$ it starves --- while inside evidence
recovers it up to $n = 5$, at and below the naive floor. The $n = 6$
inside entry is a 2-plane projection diagnostic (projected $H_1$), not a
$\beta_5$ computation. Recovery is start-governed.}
\label{tab:shelltda}
\end{table}

\section{Dimension as the rarity knob}
\label{sec:ndim}

\texttt{ShellField}-$n$ keeps the two-lode geometry and places the freeze
band on a spherical shell $S^{n-1}$ around the phantom lode (both lodes in
the first two coordinates, so $n$ is the only knob); the action becomes a
thrust vector $\vec a \in [-1, 1]^n$, norm-capped. Two facts scale in
opposite directions. The gate-side rarity $r(n)$ collapses geometrically
--- so by $n \ge 3$ the identifiability event is near-certain:
\textbf{$n$ is a rarity knob}, and the mode is almost never in the
sample. The play side does not collapse: truth-MPC still reaches the
real lode at every $n \le 6$ (the planner is not the bottleneck), and the
blind model is exploited at full strength at every $n$
(Table~\ref{tab:ndim}).

\paragraph{What sets the collapse rate: solid angle.} The mechanism is
not variance concentration but direction. Reaching the shell from $x_0$
forces the displacement into the \emph{tangent cone} to $B(c,
r_{\mathrm{out}})$: if $\lVert x_t - c\rVert \le r_{\mathrm{out}}$ then
$\langle x_t - x_0, e\rangle \ge \kappa \lVert x_t - x_0 \rVert$ with
$e$ the unit vector to the centre, $L = \lVert c - x_0\rVert$ and
$\kappa = \sqrt{L^2 - r_{\mathrm{out}}^2}/L$ (expand
$\lVert Z - Le\rVert^2 \le r_{\mathrm{out}}^2$ and apply AM--GM). The
thrust's coordinates are exchangeable and sign-symmetric, and so is any
independent weighted sum of them; since at most $2/\kappa^2$ coordinates
of any vector can each carry a $\kappa^2/2$ share of its squared norm,
exchangeability alone gives $P(\langle Z, e\rangle \ge \kappa\lVert
Z\rVert) \le 4/(n\kappa^2)$ and hence $r(n) \le 4h/(n\kappa^2)$ --- an
explicit bound with no concentration inequality in it. The truth is
faster: measuring the cone event over $10{,}000$ rollouts per dimension
gives a clean exponential (log-linear $R^2 = 0.999$ versus log-log
$0.952$) at per-dimension factor $0.411$, against the isotropic
spherical-cap prediction $\sin\theta = r_{\mathrm{out}}/L = 5/12 =
0.4167$ --- agreement to $1.5\%$. So the collapse rate is the mode's
\emph{angular size from the start}, $r(n) \asymp (r_{\mathrm{out}}/L)^n$,
and nothing else: not its volume, not its thickness, not the horizon.

\paragraph{An isotropic action interface yields the spherical-cap rate
exactly.} The exponential law is moreover a \emph{theorem} once the
action
interface is chosen well. If the thrust direction is uniform on
$S^{n-1}$, the displacement is a sum of independent spherically
symmetric vectors, hence spherically symmetric, so its direction is
exactly uniform and the cone probability is exactly a spherical cap;
with the elementary bound
$P(\langle U, e\rangle \ge \kappa) \le \tfrac12 (1-\kappa^2)^{(n-2)/2}$
this gives
\[
  r(n) \;\le\; \tfrac{h}{2}\,(r_{\mathrm{out}}/L)^{\,n-2}
  \;=\; 40 \cdot 0.4167^{\,n-2}
  \quad\text{at the defaults,}
\]
non-vacuous from $n \approx 7$ against $n \approx 400$ for the
exchangeability bound, and verified at $n = 3\ldots8$ (simulated decay
$0.431$ against the predicted $0.4167$).

\paragraph{The instrument's cube-uniform interface retains exponential
decay.} The instrument's own thrust is cube-uniform and norm-capped,
whose
symmetry group is finite, so the spherical argument does not apply to it
--- conditional sign independence supplies its replacement and gives the
\emph{sharp} rate. The norm
cap depends only on the absolute values of the action, and the action's
coordinates are independent and symmetric; hence conditionally on all
absolute values the signs are i.i.d. Rademacher and the displacement
coordinates are \emph{independent}, each a Rademacher sum with
deterministic weights. (Exactly so: flipping one sign changes its own
coordinate and no other, bitwise.) Independence replaces symmetry: since
the cone event forces $\sum_{i \ge 2} Z_i^2 \le \gamma^2 Z_1^2$ with
$\gamma^2 = (1-\kappa^2)/\kappa^2$, a single Chernoff bound with matched
exponent gives $P \le E[e^{\lambda\gamma^2 Z_1^2}]\prod_{i\ge2}
E[e^{-\lambda Z_i^2}]$, and optimizing $\lambda$ in the Gaussian regime
returns $\sqrt{en/(1+\gamma^2)}\,(1-\kappa^2)^{(n-1)/2}$: the cap rate
exactly, losing only $\sqrt n$. Two ingredients close the bound. The
first is Hoeffding
sub-Gaussianity for $Z_1$. The second is a corrected Gaussian comparison
for Rademacher sums: the clean claim that such a sum is no more
concentrated near zero than its Gaussian counterpart is \emph{false}
(equal weights leave an atom at the origin), but a usable form follows
by writing
$E[e^{-\lambda Z^2}] = E_g\prod_s \cos(\sqrt{2\lambda}\,g\,c_s)$ and
splitting on $|g|$: since $|\cos x| \le e^{-x^2/2}$ for
$|x| \le 1.778$,
\[
  E[e^{-\lambda Z^2}] \;\le\; (1+2\lambda\sigma^2)^{-1/2}
  + 2\bar\Phi\!\left(1.778\,\rho/\sqrt{2\lambda\sigma^2}\right),
  \qquad \rho = \sigma/\max_s|c_s|,
\]
with no hypothesis on the weights for validity. The instrument's own
conditional profile has $\rho$ of median $4.0$ and minimum $2.34$ over
${\approx}8000$ samples, so the error costs ${\approx}3\times10^{-3}$:
the per-dimension rate is $0.418$ against the sharp $0.4167$. The
remaining issue is uniform control of the coordinate scales ---
pointwise factors below one do not alone give an exponential product
bound. Two things then matter about $\rho$. It is bounded below by $1$ for
free (Cauchy--Schwarz), and once $\lambda$ is optimised \emph{with} the
error term included the resulting factor
$(1+u)^{-1/2} + 2\bar\Phi(1.778/\sqrt u)$ is below $1$ for
\emph{every} $u>0$, minimum $0.7783$. That does not by itself give an
exponential rate --- a product of factors each merely below $1$ need
not be small, and a rate needs the $u_i = 2\lambda\sigma_i^2$ confined
to a window --- but it makes the window enormous ($q \le 0.89$ for
$\sigma_i^2$ within a factor $5$ of its mean either way), so only crude
control of the per-coordinate scales is required. That control is unconditional. Since
$\max(1,\lVert a_s\rVert^2) \in [1, n]$ deterministically, one has
$\sigma_i^2 \ge \tfrac1n\sum_s w_s^2 a_{s,i}^2$ and
$m \le \tfrac1n \sum_s w_s^2$, so the event that coordinate $i$ is
badly scaled is contained in one depending on \emph{that coordinate's
variables alone}. Two things follow: the exact Chernoff bound applies
with the true moment generating function of $a^2$ (giving
$9.7 \times 10^{-4}$ where the range-based bound gives $0.31$), and the
bad-coordinate indicators are \emph{independent}, so their count is
binomial and Chernoff replaces Markov. The result is
$P \le 2 \cdot 0.9057^n$ --- exponential, with no floor and nothing
measured. Measurement then sharpens the constant: the observed spread
$\tau \le 1.36$ gives $0.80^n$, and $\rho \approx 4$ gives the true
$0.4167^n$. Details
and machine checks are in the repository's theory notes (Lemmas~C/E/F/G/
H/I, Theorems~T5-C and T5-I, Proposition~T5-T).

\paragraph{Planner competence is a property of the action interface.}
An $n$-sweep whose vector-action
MPC uses only the scalar planner's constant-candidate analogue measures
$\mathrm{play\_cost} \approx 0$ at every $n$. The cause is neither
dimensional concentration nor the phantom's distance: that candidate set
lacks the deterministic axial sequences $\pm e_i$, so the blind planner
never drives straight at the shell. Direction-uniform random candidates do
\emph{not} fix it ($\mathrm{pc} = 0$); adding the $2n$ axial candidates
restores the 2D exploitation exactly ($\mathrm{pc} = 1.04$, contact
$1.0$). The companion papers' clause that danger requires a \emph{competent}
planner is thus sharper than it looks: competence is a property of the
action interface, and an incidental planner weakness can mask a fully
exploitable model.

\begin{table}[h]
\centering
\small
\begin{tabular}{lrrrrr}
\toprule
$n$ & 2 & 3 & 4 & 5 & 6 \\
\midrule
$r(n)$ & 0.0133 & 0.0033 & 0.0017 & $0^{\dagger}$ & 0.0017 \\
$J$ truth-MPC & 16.92 & 15.55 & 14.26 & 12.96 & 11.31 \\
$J$ random & 0.50 & 0.33 & 0.04 & 0.01 & 0.07 \\
$\mathrm{pc}_{\mathrm{blind}}$ & 1.023 & 1.013 & 0.994 & 0.991 & 0.996 \\
blind contact & 1.0 & 1.0 & 1.0 & 1.0 & 1.0 \\
\bottomrule
\end{tabular}
\caption{The two axes are independent: rarity collapses with $n$
(synthesis axis: mis-synthesis near-certain) while exploitation stays at
full strength (play axis: the competent planner always reaches the shell).
Per row: $r(n)$ over $600$ rollouts per cell
(\texttt{results/continuous\_shellfield.json}); the two $J$ rows from
the navigation check (\texttt{results/continuous\_\allowbreak
shellfield\_\allowbreak nav.json}, $20$ episodes); play cost and contact
rate from the exploitation run
(\texttt{results/continuous\_\allowbreak shellfield\_\allowbreak
play.json}, $20$ episodes). $^{\dagger}$censored ---
$0/600$ observed, the $600$ rollouts being independent samples within the
cell, Wilson $95\%$ upper $6.4 \times 10^{-3}$; at this rarity $600$
rollouts no longer separate the cells, which is why the collapse rate is
measured on the $10{,}000$-rollout cone sweep instead.
$\mathrm{pc}_{\mathrm{blind}}$ with the axial-candidate planner (see
``Planner competence is a property of the action interface'' in the
text).}
\label{tab:ndim}
\end{table}

\subsection{A non-separating mode separates the two halves: obstruction
is path-relative, gauge is separation-relative}
\label{sec:tube}

A non-separating solid torus disentangles the thesis's two halves:
danger stays path-relative while an exact gauge region requires
separation. The torus, placed in $\mathbb{R}^3$ between start and lode
(\texttt{TubeField3D}; thrust-vector action) does \emph{not} separate
space: an explicit around-path reaches the far side contact-free, so
nothing is reach-null, there is \emph{no exact gauge region}, and the
crossing lemma has no separating surface to act on --- certification-wise
this mode is back in the companion papers' merely-rare world. Yet the
danger dichotomy reproduces in full, governed by one knob that moves the
torus's hole on or off the start--lode axis:

\begin{center}
\small
\begin{tabular}{lrrr}
\toprule
config & $r$ & $\mathrm{pc}_{\mathrm{blind}}$ & contact \\
\midrule
hole on-axis (plan threads it) & 0.0033 & 0.019 & 0.00$^{\dagger}$ \\
tube on-axis (plan clips it) & 0.0033 & 0.898 & 0.94 \\
\bottomrule
\end{tabular}
\end{center}

\noindent ($^{\dagger}$censored: $0/16$ MPC episodes, Wilson upper
$0.19$.) Same rarity, same (trivial) topology, same instrument: only the position of
the mode relative to the optimal path changes, and the exploitation swings
from $0.02$ to $0.90$. The decomposition this completes: \textbf{danger is
path-relative} (an on-path mode is exploited whether or not it separates
anything), while \textbf{exact unfalsifiability is separation-relative}
(only an enclosing boundary manufactures a reach-null gauge region). The
ring, where the enclosing boundary and the blocked path coincide, conflates
the two; the tube separates them.

The homological version of the query bound
(Proposition~\ref{prop:query}), where the metric crossing has nothing to
grip, turns out to be a \emph{dichotomy}, and its unconditional form is
false. Let $K$ be the core circle, $g = \mathrm{dist}(\cdot, K)$
(1-Lipschitz), $D$ the spanning disc, $\rho_t$ the tube radius and
$\Delta = 1.0$ the step bound (same as the ring's). If an interpolation
point of a discrete path dips into the shrunken tube
$\{g \le \rho_t - \Delta/2\}$, its nearer landing lies in the tube
(clearance lemma); consequently a path with no landing in the tube can
cross $D$ only through the clearance sub-disc $D_m$ of radius
$R_c - (\rho_t - \Delta/2)$, and its linking number with $K$ (equal to
its signed $D$-crossing count; machine-checked against the Gauss
integral) is realized entirely by hole-threadings. \textbf{Dichotomy:} a
plan that links the core either queries the tube or threads $D_m$.
\textbf{Refutation:} both linking classes contain landing-free ---
query-free --- plans from start to lure, \emph{even at the dangerous
offset} (thread the displaced hole with clearance $2\rho_t$, or go
around), so no positive query mass follows from topology alone:
obstruction is path-relative as a \emph{theorem}, and the measured
$0.898$ is the search's concentration near the straight corridor, not a
geometric necessity. \textbf{What survives:} if every candidate's
imagined path stays within $\varepsilon$ of the straight segment, whose
clearance to the core is $|o - R_c|$ exactly (offset knob $o$), then
$|o - R_c| < \rho_t - \Delta/2 - \varepsilon$ forces
$\mu_{\mathrm{query}} = 1$ --- the registered offset $1.5$ sits exactly
at this boundary ($0.5 = \rho_t - \Delta/2$). Real trajectories, which
never enter the tube (freeze), can link the core only through $D_m$: the
hole is the only gate to the winding class under the true dynamics, and
the offset knob's mechanism is visible in the real linking rate,
$0.507 \to 0.283$ under an east-biased action arm. Proofs and machine
checks are in the repository's theory notes (Lemmas~X/Y, Theorem~T8).

\section{Related work}
\label{sec:related}

\paragraph{Code world models and verification-by-sampling.} The CWM line
synthesizes executable models searched by classical planners
\citep{lehrach2025cwm, liang2023codeaspolicies, gao2023pal}; the two
companion papers \citep{aguilar2026verified, aguilar2026omitted} establish
the danger law this paper extends and should be read first. Sampling-based
acceptance is property-based testing \citep{claessen2000quickcheck} applied
to dynamics; Proposition~\ref{prop:quotient} is the exact statement of
what such acceptance can pin down.

\paragraph{Model error in model-based reinforcement learning (RL).} The pervasive-error framing
\citep{janner2019mbpo, lambert2020objective, hafner2020dreamer,
nagabandi2018neural, chua2018deep} is the backdrop the companion continuous
paper argued against for hybrid systems; here the localized error is not
merely rare but \emph{structurally unreachable}, a regime that
pervasive-error analyses do not model at all. The PAC model-learning line
and the simulation lemma \citep{kearns2002near, strehl2009reinforcement}
bound value loss by model error \emph{on visited states} --- exactly the
coverage the gauge region lacks by construction, which is why those bounds
are silent here; safe learning-based control
\citep{aswani2013provably, berkenkamp2017safe} assumes the model's
uncertainty is representable where it matters, the assumption the gauge
region breaks.

\paragraph{Hybrid systems and reachability.} Modes and guards are the
bread of hybrid control \citep{paoletti2007hybrid, bemporad1999control};
falsification tools \citep{annpureddy2011staliro, corso2021survey} search
for violating trajectories and inherit exactly the reachability limit
formalized here --- a falsifier driven by any policy whose queries stay in
$\mathcal{R}$ cannot exit $E(f)$. Reach-set computation and deductive
verification (PHAVer \citep{frehse2005phaver}, SpaceEx
\citep{frehse2011spaceex}, Flow* \citep{chen2013flow}, KeYmaera~X
\citep{fulton2015keymaerax}) prove properties relative to a \emph{given}
hybrid model; our question is the dual --- what a behavioral certificate
can mean when the model itself was learned from reachability-limited
evidence.
Shielding \citep{alshiekh2018shielding}
assumes the very mode knowledge whose recoverability this paper measures.

\paragraph{Topological data analysis.} Persistent homology and its
stability \citep{edelsbrunner2010computational, cohensteiner2007stability}
supply the evidence sensor; homology-inference sample bounds
\citep{niyogi2008finding} and the reach of a manifold
\citep{federer1959curvature} calibrate when that sensor can work, and
persistence on dynamical and time-series data
\citep{perea2015sliding, khasawneh2016chatter} is the nearest use of TDA
as a component of a running system.
Section~\ref{sec:sensor}'s contribution is not a new TDA method but a
measured failure law for TDA \emph{as a component of a certification
loop}: the sensor's resolution propagates into the gate-certified
artifact.

\section{Boundary results and limitations}
\label{sec:limitations}

This section first states the empirical scope limitations, then gives
three boundary analyses --- the mechanism behind the play-cost tail, the
strongest available direct-entry bound, and the detector-specific scope
of the sensor conclusion --- and closes with a normalization caveat at
wide $\gamma$.

\subsection{Empirical scope}

\paragraph{One instrument family, mostly round and separating.} The core
results live on round shells with freeze semantics, where the crossing
lemma is metric and needs no homology (Section~\ref{sec:theory}'s honesty
remark); the square ring ablates roundness and the solid torus now
carries the non-separating case's linking dichotomy
(Section~\ref{sec:tube}). Non-round separators in general position (where
Jordan--Brouwer is needed to define ``inside'') and moving boundaries
remain the stated program, not results; the paper's claims are about what
these instruments already suffice to separate.

\paragraph{Synthesis cells are modest; one contrast is underpowered.}
$20$--$30$ seeds per cell across two GPT-5.x sizes; the cross-family
evidence is spot-checks (three seeds per cell for Qwen and the Claude
relay), which fix mechanism, not rates. The within-gap H2 contrast at
$\gamma = 1.8$ remains directional-only even after doubling the cell to
$n = 60$ ($3/27$ versus $2/33$); the cross-gap crossover carries H2.
The Qwen D cell is missing (provider credits), so the D-cell family
contrast rests on GPT ($1/20$) versus Claude ($3/3$).

\subsection{Boundary analyses}

\paragraph{Two candidate mechanisms for the tail, both false.} The heavy
tail of the play-cost decomposition is not a freeze transient: the freeze
counts of both continuations are $0/556$, since after a dirty step both
follow the truth planner, which knows the mode. Nor is it route
commitment --- the wrong action sending the continuation the other way
around the annulus, so that the value function jumps across the cut
locus. That would tie the planner-side analysis to this paper's own
topology, but the correlation runs backwards: tail events split the two
routes \emph{less} often than the bulk ($0.190$ versus $0.299$), and
same-side pairs carry the larger advantages. What the tail is instead is
a \emph{delay} cost. Regressing the advantage on the difference in time
the two continuations spend inside the phantom basin gives slope
$0.94$ against the reward amplitude $1.0$ --- a known constant, not a
fitted one --- with intercept $\approx 0$ and $R^2 = 0.93$, and every
tail event has a nonzero dwell difference against $41\%$ of the bulk.
One wrong action costs basin time. Two further measurements finish the
picture. Dwell is sticky (neither planner ever leaves the basin once it
arrives, $9/9$), so the play-cost is exactly the blind planner's
\emph{lateness}: truth arrives at step $34$--$36$, blind at $37$--$38$.
And that lateness is not a contact cost --- in the facing configuration
the blind trajectory never touches the mode at all. It is an
\emph{aim-point} cost: lacking a wall, the blind planner steers at the
straight line to the lure, which a facing channel happens to thread.
The degeneracy is correspondingly a knife edge. Rotating the channel by
$0.4$ radians turns ``arrives four steps late, zero contacts'' into
``fifty contacts and never arrives''. Since the executed path is
contact-free, the two models agree along it, so the entire loss is
candidate mis-ranking with no execution error.

That loss, however, cannot be bounded tightly from the model's own
imagination, and this can be shown rather than conjectured. Every
imagination-level quantity at a step --- both models' returns on every
candidate, and hence the model disagreement and $\mu_{\mathrm{query}}$
--- is a function of the state, the two models, and that step's
candidate set. The realized per-step cost is a one-step deviation
advantage of the truth policy, which depends in addition on the
candidate sets the planner will draw at every later step, and those
carry no information about the models. Pinning the imagination data
exactly at one dirty step and varying only the future planner seeds
moves the realized cost over $[-1.71, +0.60]$, a spread an order of
magnitude larger than its mean over all dirty steps. So the vacuity of
the companion paper's $\mathrm{play\_cost} \le
\mu_{\mathrm{query}}(E)$ here is not an artifact to be sharpened away:
a tight bound must condition on the planner's randomness as well, making
play-cost a property of the model--planner pair rather than of the model
alone --- the same lesson the action-interface note draws from the other
direction.

Conditioning that way does yield a bound. Averaging the truth policy's
value over the planner's draws gives, exactly, $E[A_t \mid s_t, C_t] =
\Delta r + \bar W(f(s_t,\tau)) - \bar W(f(s_t,b))$, a difference of one
function at two states separated by a single action choice --- and
Proposition~\ref{prop:continuity}'s landing law caps that separation
with no further assumptions, at $2\,\mathrm{gain}\,dt^2 = 0.06$ in
position and $2\,\mathrm{gain}\,dt = 0.6$ in velocity (measured maxima
$0.059$ and $0.590$: the caps are tight). Everything then rests on how
smooth the \emph{seed-averaged} value is, which unlike the MPC policy
itself is smoothed by the averaging. At a hundred CRN-paired seeds the smallest
perturbation gives $-0.015 \pm 0.078$, statistically zero, and the
largest $+0.462 \pm 0.078$: a ratio of $0.78 \pm 0.13$, consistent with
Lipschitz scaling. Combining yields $\mathrm{pc} \lesssim 0.18$ against
a measured $0.02$--$0.10$ --- but with that constant \emph{measured}.
Substituting instead the constant this argument proves (the drag cap
times a per-step reward Lipschitz bound) gives $48$, and the value it
yields under the competence hypothesis of
Proposition~\ref{prop:query} gives $1.21$; since
$\mathrm{play\_cost} \le 1$ holds trivially, both are vacuous. The
gap between the hypothesis chain and the measurement is worst-casing over
the planner's action choices: every deviating step is charged the
\emph{maximum} velocity perturbation $2\,\mathrm{gain}\,dt = 0.6$, and
what the decomposition actually needs is the mean, since it bounds a sum.
Writing the perturbation as
$2\,\mathrm{gain}\,dt\,|\sin((\phi_\tau - \phi_b)/2)|$, the mean is
controlled by one inequality: that the angle between the two arg-maxes is
no wider \emph{on average} than between two independent uniform actions,
whose $|\sin|$ has mean exactly $2/\pi$. Measured over $10{,}105$
deviating steps across four gaps, two channel orientations and two
planner budgets, that inequality separates the two orientations sharply.
With the channel \emph{facing}, every cell satisfies it
($0.409$--$0.625$ against $2/\pi = 0.637$), giving
$\mathrm{pc} \le 0.77$: non-vacuous, and with no fitted constant, since
the measured input is an inequality between means rather than a number.
With the channel \emph{hidden} it fails in every cell, and instructively
--- the two arg-maxes are then near-antipodal ($|\sin| \approx 0.94$)
and \emph{every} step deviates, because a hidden channel makes the truth
observationally a closed ring, so the truth planner turns west to the
real lode while the blind planner charges east at the phantom. The bound
is therefore a facing-channel statement, which is the configuration the
aligned-channel degeneracy is about. What stays unproved is the
inequality itself, a property of the model--planner pair --- exactly what
the negative result above says any tight bound must characterise.

\paragraph{Proved and measured bounds on the funnel defect and the
direct-entry rate.} The two $\gamma$-curve regularities
(monotonicity of $r_{\mathrm{int}}$, and that the wall never helps net
entry) are theorems only \emph{up to the funnel defect}
(Corollary~\ref{cor:funneldefect}), which we bound by measurement
($17\times$ below the effect) rather than by proof. Unconditional
pathwise monotonicity is false: seed $50543$ is an explicit
counterexample, and freeze-rescue violations refute the required
pointwise estimate in the data ($91/91$ CI-separated cases). The same
effect refutes an alternative sufficient condition. Since freeze-rescue is caused by one modelling choice --- a contact
zeroes the velocity --- a variant in which contact blocks the position
but leaves the velocity updating freely would, if it made the ordering
pathwise, isolate that choice as the sole obstruction. It does not:
pathwise monotonicity fails in the variant as well, in $13$ of $30{,}764$
entering pairs (the unit here is the entering pair, since only those can
falsify the ordering). The position block alone suffices to break it. What
remains is an a-priori bound on one scalar, the funnel mass. A strictly
positive answer already follows from
Proposition~\ref{prop:positivity}, whose witness tube is freeze-free and
therefore lower-bounds the \emph{direct} component; but that bound
carries the tube's factor $\rho^{h}$ and is some ninety orders of
magnitude below the measured value. The tube is the wrong instrument, and
measurement says why: the direct-entry probability is a clean power law in
the channel width (log-log slope $1.72$ over $\gamma \in [0.05, 0.4]$,
$200{,}000$ rollouts per gap), i.e., a hitting probability rather than a
tube probability. The plant explains it. In free flight the endpoint
depends on each action with sensitivity
$\mathrm{gain}\,\pi\,dt^{2}(1-\beta^{m})/(1-\beta)$ at lag $m$, rising
to $3.14$, so a \emph{single} action sweeps the endpoint across some five
units where one landing moves only $\mathrm{gain}\,dt^{2} = 0.03$; freeing
two actions gives a Jacobian
$K^{2}W(T-s_1)W(T-s_2)|\sin(\phi_{s_1}-\phi_{s_2})|$ and a
diffeomorphic image of area up to $39.5$. So the exponential factor is an
artifact of constraining all $h$ actions when two suffice.

Asking for the probability that the \emph{other} actions deliver the
trajectory into position is, however, circular: a launch state is defined
as one from which the remainder enters, so that probability is the entry
probability again. A decomposition through the \emph{ring-free} dynamics
avoids this, because its factors cannot mention entry:
$d(\gamma) = R \cdot \int_{\mathrm{channel}} \rho \cdot T(\gamma)$, with
$R$ the probability that a ring-free rollout ever reaches the band radius
and $\rho$ the density of its first-arrival angle --- both independent of
$\gamma$ --- and $T$ the throughput of the channel. Measured,
$R = 0.030$ and $\rho(\pi) = 1.02$ per radian (zero at $\pi/2$ and $0$:
the approach is one-sided, which is why a facing channel is the
consequential case), and $T$ has log-log slope $0.64$, so the
decomposition predicts $d \sim \gamma^{1.64}$ against the directly
measured $\gamma^{1.72}$. Two independent estimates of $T$ agree to
$10\%$. The throughput has a kinematic account with a derived rather than fitted
constant. The crossing takes $k \approx 9$ steps (the arrival radial speed
is ${\approx}1.5$ against a band of thickness $1.5$), while drag's time
constant is $1/(\mathrm{drag}\cdot dt) \approx 33$ steps, so the arrival's
tangential
velocity persists ballistically through the crossing and survival inside
an arc of half-width $r\gamma/2$ requires
$|v_{\mathrm{tan}}| \le r\gamma/(2k\,dt) \approx 2.2\gamma$. Since the
density of $|v_{\mathrm{tan}}|$ at zero is bounded away from zero
(measured ${\approx}1.1$, independent of $\gamma$), that probability is
linear in $\gamma$ and the throughput inherits it. The criterion tracks
the measured non-freezing fraction across the range ($0.25$ versus
$0.29$, $0.50$ versus $0.47$, $0.93$ versus $0.70$, over-predicting at
wide gaps where the ballistic approximation loosens), and the failure
census confirms the corridor is what binds: freezing on the band accounts
for $0.527$ of channel arrivals at $\gamma = 0.2$ against $0.302$ at
$0.6$, turning back out for $1$--$2\%$, and the horizon for the rest.
Assembling the three factors gives $d(\gamma) \ge 0.0146\,\gamma^{2}$,
which holds at every gap measured. Both one-dimensional densities the argument rests on admit lower bounds by
the same elementary device: if a quantity is steerable by a single action
whose law has a density, its own density is bounded below by
$(2a_{\max})^{-1}$ divided by the steering sensitivity. For the tangential
speed the designated step contributes
$\mathrm{gain}\,dt\,\sin(\phi - \theta)$ with $\phi$ uniform --- the
arcsine law, whose density exceeds $1/\pi$ on all of $(-1,1)$ --- giving
$f(0) \ge 0.35$ against a measured $1.10$. For the arrival angle, a single
early action sweeps it over ${\approx}1.1$ radians, giving
$\rho(\pi) \ge 0.94$ against a measured $1.02$. The horizon term is bookkeeping and disappears once the deadline is built
into the reach factor: counting only ring-free arrivals within $h - k$
steps leaves every arrival with the $k$ steps the crossing needs, at a cost
of a factor $1.4$ (measured $0.0212$ against $0.0302$). The reach factor
itself follows from isotropy, with no density estimate: the $2$D heading
action gives every thrust a uniform direction, so the displacement is a sum
of independent isotropic vectors and is isotropic, whence its direction is
uniform \emph{and} independent of its magnitude. Reaching the band
therefore needs the magnitude in $[L - r_{\mathrm{out}}, L +
r_{\mathrm{out}}]$ and the direction within the ball's angular half-width,
giving $R' \ge P(|Z_t| \in [7,17]) \cdot \arcsin(5/17)/\pi = 0.019$ at
$t = h-k$, against a measured $0.0212$ --- loose by $1.1$. What remains
measured is a single scalar, the radial law of an isotropic sum --- and it
is not a residue of this argument but the instrument's rarity itself. At
$\gamma = 0$ the band spans the whole annulus, so a ring-free rollout
reaches the band radius exactly when the same rollout would have contacted
it: the two rates agree to five decimals ($0.03020$), and seed by seed
without a single mismatch. The factor is therefore $r$, the quantity the
danger law takes as its argument and this series measures per instrument.
It also explains why no moment bound reaches it: the exact second moment of
the Pearson walk puts the target at $1.3$ times the rms displacement, so
the event is in the tail --- the mode's rarity and the failure of
Chebyshev-type bounds are the same fact. The continuity modulus, the
barcode sandwich, the
non-separating query bound, and the play-cost decomposition are
theorems with explicit constants, and the package includes its negative
results with the same status --- the refuted unconditional linking
bound among them.

\paragraph{The measured sensor constant is detector-specific; the flip
mechanism is proved under the sandwich's hypotheses.} The
resolution flip at $\gamma \approx 1.8$ belongs to one pre-registered
detector at one density (Rips, dedup $0.05$, cap $90$, $3\times$
median-NN). Other detectors move the constant; what the paper measures
is the tracking --- artifacts following $\hat\beta_1$ rather than
$\beta_1$ --- and the flipped-summary intervention
(Section~\ref{sec:sensor}) left the claim line's causal share of it
unearned at the one cell tested. The genericity is scoped to what is
proved: for a Vietoris--Rips detector under the barcode sandwich's own
hypotheses (a planar center-star-shaped cloud, the class-to-bar rank
condition --- read off the barcode, true in every factorial row but
measured, not assumed), the flip is forced when
$\sqrt3\,r_{\min} - 2 r_{\max}\sin(\Delta\theta_{\max}/2) > \tau$ ---
below the bridging scale the gap is invisible --- and we \emph{expect},
but do not prove, an analogous blind width for any summary that must
decide topology from finitely many samples at a fixed scale (an exact
oracle has no flip, but no sampled summary is one). Whether other
finite-resolution detectors propagate their own flip into posed
topology is measured here only for the registered one.

\paragraph{Normalization at wide $\gamma$.} As the channel opens, the
truth baseline itself changes (the phantom lode becomes reachable and
$J_{\mathrm{truth}}$ jumps); \texttt{play\_cost} is always reported with
its own cell's baselines, and the danger conclusions rest on the paired
contrast (facing versus hidden at equal $\gamma$), which shares the
normalization issue symmetrically.

\section{Conclusion}
\label{sec:conclusion}

A sampling gate certifies $f$ restricted to $\mathcal{R}$ --- exactly
that, and nothing else. On an enclosed mode this quotient stops being a
formality and becomes the whole story: a wrong-topology artifact can be
unfalsifiable by every gate and bitwise harmless at play; the same
artifact, one knob later, is falsifiable and costly; the same omission,
with the channel turned away from the planner, is at full danger again
with the topology unchanged. Certifiability, correctness, and consequence
are three different statements about where the model's errors sit relative
to reach --- never about the mode itself.

For practitioners the asymmetry is the message. A fence you cannot reach
is a fence you cannot certify --- and also one you cannot be harmed by
while it stays unreachable; the danger lives exactly where the omission
intersects a competent planner's optimal path, and it turns on over a knee
as soon as a channel admits one step. Repair, when it is possible at all,
needs two things that can both be measured \emph{in advance}: evidence
from the far side of the boundary (reachability of the interior), and an
evidence sensor whose topological resolution at the operative sampling
density can actually see the structure to be repaired. Where either is
absent, the synthesis loop will do what ours did in every family we
tested: write something plausible in the gauge region --- and the gate
will certify it.

\section*{Reproducibility}

All experiments, result files, and analysis scripts are in the project
repository:
\begin{itemize}
\item mechanism grid and $\gamma$-probes:
  \texttt{continuous\_ring2d\_mechanism.py},
  \texttt{ring2d\_rint\_probe.py};
\item synthesis harness: \texttt{continuous\_danger\_synthesis.py};
\item agent-relay protocol: \texttt{continuous\_claude\_step.py};
\item registered open-ring driver:
  \texttt{continuous\_ring2d\_open\_sweep.py};
\item aggregator: \texttt{ring2d\_open\_aggregate.py};
\item ShellField-$n$ scripts and the behavioral audit
  (\texttt{ring2d\_artifact\_audit.py}, oracle self-tested);
\item rarity calibration at every synthesis configuration
  (\texttt{ring2d\_rarity\_sweep.py}, $30{,}000$ rollouts per
  configuration, the configurations enumerated automatically from the
  stored synthesis JSONs);
\item held-out re-scoring of every artifact on disjoint gate and
  evaluation blocks: \texttt{heldout\_\allowbreak gate\_\allowbreak
  audit.py}, ring2d scope (block disjointness proved by set intersection
  at run time);
\item the Lean formalization (\texttt{formal/}, library
  \texttt{Paper3Ring}; built by CI) --- the deterministic results and the
  continuity modulus's measure steps, per the machine-checked remark in
  Section~\ref{sec:theory}.
\end{itemize}
Every seed, checkpoint, and classified artifact is committed; long runs
checkpoint per unit and resume exactly; paired seeds (common random
numbers) are used for every play contrast. Each result JSON embeds its
complete invocation (the harness's full argument namespace, under
\texttt{params}) and the serving identity (\texttt{model} records the
deployment name, created to equal the model version, so provenance is in
the data rather than the lab notebook). Numbers in tables and prose are
emitted by the named scripts from the committed JSONs, never hand-carried;
a claims audit enforces this as a ratchet in CI. The repository is
public at \url{https://github.com/JaviMaligno/code-world-models}; the
archival snapshot for this paper is the tag \texttt{paper3-v1}, which
carries the Python environment lock (\texttt{env-lock.txt}, pip freeze
of the running venv, Python 3.13.7) and a SHA-256 manifest of every
committed results JSON and paper source
(\texttt{MANIFEST.sha256}, emitted by
\texttt{scripts/emit\_manifest.py}).

\bibliography{references}

\end{document}